\documentclass{article} 
\usepackage{iclr2027_conference,times}

\usepackage{amsmath,amsfonts,bm}

\def\eqref#1{equation~\ref{#1}}

\def\1{\bm{1}}

\DeclareMathAlphabet{\mathsfit}{\encodingdefault}{\sfdefault}{m}{sl}
\SetMathAlphabet{\mathsfit}{bold}{\encodingdefault}{\sfdefault}{bx}{n}

\usepackage{hyperref}
\usepackage{url}
\usepackage{tikz}
\usetikzlibrary{
    arrows.meta,
    positioning,
    shapes.geometric,
    fit,
    backgrounds
}
\usepackage{graphicx}
\usepackage{xcolor}
\usepackage{caption}
\usepackage{subcaption} 
\usepackage{amsmath}
\usepackage{amsthm}
\usepackage{booktabs}
\usepackage{tabularx}
\usepackage{array}
\usepackage{enumitem}
\usepackage{multirow, makecell}
\usepackage{soul}
\usepackage{amssymb}
\usepackage{comment}
\usepackage{natbib}
\usepackage{pifont} 
\usepackage{makecell}

\newtheorem{theorem}{Theorem}
\newtheorem{lemma}{Lemma}

\newtheorem{corollary}{Corollary}

\theoremstyle{definition}

\theoremstyle{remark}

\title{Evolving Hybrid Quantum–Classical Architectures for Image Classification}

\author{Devroop Kar, Daniel Krutz \& Travis Desell \\
Rochester Institute of Technology\\
Rochester, NY 14623, USA \\
\texttt{\{dk7405,dxkvse,tjdvse\}@rit.edu}
}

\iclrfinalcopy 
\begin{document}

\maketitle

\begin{abstract}


Hybrid quantum classical neural networks integrate parameterized quantum circuits (PQCs) with established deep learning architectures, but their performance depends strongly on the choice of quantum circuit architecture, a choice that remains largely manual. Most existing approaches rely on hand-designed or fixed circuit ans\"atze, requiring circuit structure, gate composition, and qubit connectivity to be specified in advance with no guarantee that they suit the task. This limitation is especially acute in image classification, where quantum circuits must transform features extracted by classical networks while remaining compact enough for practical training, requirements that generic, task-agnostic ans\"atze are unlikely to satisfy simultaneously. We extend EXAQC, an evolutionary framework for automated quantum circuit discovery, to image classification. EXAQC evolves PQCs as intermediate processing modules while retaining classical feature-extraction and prediction layers. On MNIST, Fashion-MNIST, and CIFAR-10, EXAQC achieves 98.42\%, 90.62\%, and 85.47\% accuracy, respectively, while using comparable gate counts to other quantum architecture-search methods. Against classical networks, evolved hybrid models maintain comparable accuracy with substantially fewer trainable parameters, reaching 85.68\% on CIFAR-10 with over 25$\times$ fewer parameters than a 10-layer CNN. Encoding choice also matters: rotation-based encodings (RX, RY, U3) outperform amplitude encoding by 22-25 points on CIFAR-10. These results demonstrate that automated circuit discovery yields compact quantum modules that can replace larger classical components in vision architectures while retaining competitive accuracy.

\end{abstract}
\section{Introduction}
\label{sec:introduction}

Parameterized quantum circuits (PQCs) are widely studied for quantum machine learning (QML), where classical data is encoded into quantum states and processed by trainable quantum operations \cite{havlivcek2019supervised,schuld2020circuit,perezsalinas2020data,wang2024trainability,ding2024scalable,tudisco2026evaluating}. However, the effectiveness of the concerned circuit depends strongly on architectural choices such as gate composition, circuit depth, connectivity, and data encoding \cite{sim2019expressibility, hubregtsen2021evaluation,benitez2025bayesian,liu2025analysis}. Poor circuit design can also produce barren plateaus, where gradients vanish exponentially with qubit count~\cite{mcclean2018barren, cerezo2021cost}, stalling training and rendering even expressive circuits effectively untrainable. Architecture choice therefore determines whether a hybrid model can learn at all. Selecting an effective PQC architecture therefore remains an important challenge. The difficulty of manually selecting circuits that are both well-suited to the task and trainable has motivated quantum architecture search (QAS), which automatically optimizes circuit structure through approaches including differentiable, predictor-based, and evolutionary search \cite{zhang2020dqas,wu2023quantumdarts, zhang2021predictor,ding2022evolutionary,zhang2023evolutionary,kar2026investigating}. 



Evolutionary circuit search has shown promise on low-dimensional benchmarks~\cite{kar2026evolutionary}, but practical machine learning rarely operates at that scale. Image classification offers a more demanding test: MNIST and Fashion-MNIST contain 784 input values per image, and CIFAR-10 contains 3,072, far more than the handful of qubits available on near-term hardware. Hybrid architectures bridge this gap by using classical feature extractors to compress images into compact representations for quantum processing ~\cite{wang2022development,liu2021hybrid}. As a result, a harder dataset does not require a larger quantum circuit, since the circuit's input can stay small even as the source images grow more complex. The question is therefore not whether evolved circuits can process thousands of raw pixels, but whether they can still find useful structure in compressed features as the underlying task becomes harder. Due to space constraints, a detailed discussion of related work on hybrid quantum classical image classification, quantum architecture search, and evolutionary circuit design is provided in Appendix~\ref{sec:related_work}.

This work extends EXAQC~\citep{kar2026investigating}\footnote{Source code publicly available at: \url{https://github.com/travisdesell/exaqc}} to image classification, introducing hybrid quantum circuits using encoders to extract image features. This works also uses U3 gates as a quantum encoding mechanism unlike previous works (Refer Appendix \ref{sec:related_work}). Because circuit structure is discovered rather than fixed in advance, the evolved circuits can be tailored to the learned features and kept compact, avoiding both mismatched ans\"atze and the trainability problems of larger circuits. We investigate both scalability and whether evolved quantum modules can replace larger classical components while maintaining competitive predictive performance with fewer trainable parameters. Specifically, we address the following research questions: 

\begin{itemize}
    
    \item \textbf{RQ1: Scalability and Hybrid Effectiveness.}
    Can EXAQC scale from low-dimensional learning problems to complex image-classification tasks by discovering useful quantum transformations of learned image representations within a hybrid quantum classical framework? \textit{Yes: EXAQC reaches 98.42\%, 90.62\%, and 85.47\% on the full 10-class MNIST, Fashion-MNIST, and CIFAR-10 tasks, exceeding the compared architecture-search methods on CIFAR-10 while using fewer gates than EQNAS and AQEA-QAS.}

    \item \textbf{RQ2: Representation Dimensionality.}
    How does increasing the dimensionality of the quantum circuit inputs affect predictive performance, circuit complexity, convergence, and computational requirements? \textit{The quantum representation clearly affects performance and circuit size, since CIFAR-10 required more qubits (8-10), larger circuits (42-60 gates), and a larger search budget.}


    \item \textbf{RQ3: Quantum Encoding Effectiveness.}
    How does the choice of quantum input encoding (amplitude, RX, RY, or U3) affect the accuracy and complexity of evolved circuits, and does this effect depend on the dataset? \textit{The effect is strongly dataset-dependent: encodings differ by under one point on MNIST, whereas on CIFAR-10 rotation-based encodings (83.75-86.68\%) far outperform amplitude encoding (61.38\%), with the best-performing U3 also producing the largest circuits.}

    \item \textbf{RQ4: Parameter Efficiency.}
    Do evolved quantum components maintain competitive classification performance relative to larger classical counterparts, despite having fewer trainable parameters? \textit{Yes: EXAQC matches the 10-layer CNN on CIFAR-10 (85.68\% vs.\ 85.50\%) with 96.1\% fewer parameters and exceeds the 1-layer CNN on Fashion-MNIST (91.11\% vs.\ 89.47\%) with 54.8\% fewer parameters.}
\end{itemize}


\section{Methodology}
\label{sec:methodology}


Raw images are far too large to encode directly on near-term quantum hardware, and fixed circuit ans\"atze may be poorly matched to the features a classical network extracts. To address both issues, we extend EXAQC to higher-dimensional image classification using a hybrid quantum classical architecture consisting of a classical encoder\footnote{Please note that encoder in the context of this work means feature extractor}, an evolved parameterized quantum circuit (PQC), and a classical decoder. The encoder extracts a compact image representation, EXAQC evolves the quantum transformation applied to this representation, and the decoder maps quantum measurements to class predictions. This design separates the dimensionality of the original image from that processed by the quantum component, allowing a small, evolved circuit to operate on complex image data.

\subsection{Hybrid Quantum Classical Architecture}
\label{sec:hybrid_architecture}

\begin{figure}
    \centering
    \includegraphics[scale=0.17]{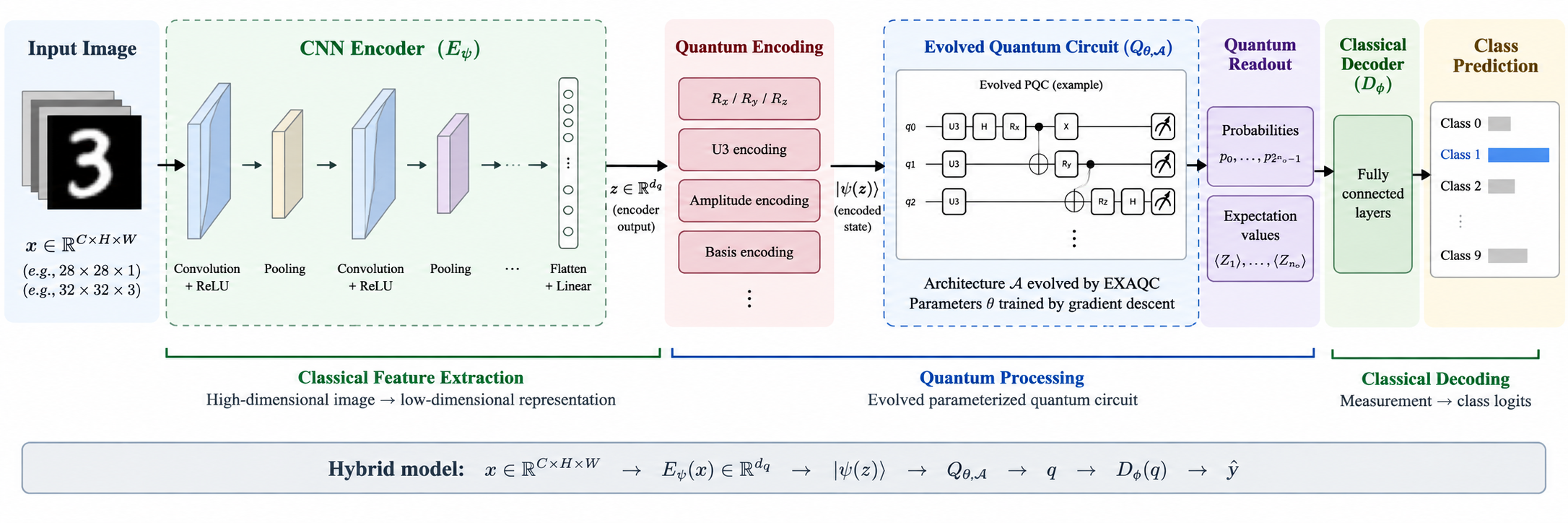}
    \caption{Overview of the proposed EXAQC hybrid quantum classical architecture. A CNN encoder compresses the input image into a lower-dimensional representation that is encoded and processed by an evolved PQC. Quantum measurements are subsequently mapped to class
    predictions by a classical decoder.}
    \label{fig:exaqc_hybrid_architecture}
\end{figure}

Given an image $\mathbf{x}$, the hybrid model is $\hat{\mathbf{y}} = D_{\phi} \left(Q_{\theta,\mathcal{A}} \left(E_{\psi}(\mathbf{x})\right) \right)$, where $E_{\psi}$ is the classical encoder, $Q_{\theta,\mathcal{A}}$ is a quantum circuit with trainable parameters $\theta$ and evolved architecture $\mathcal{A}$, and $D_{\phi}$ is the classical decoder. For $\mathbf{x}\in\mathbb{R}^{C\times H\times W}$, the encoder maps the raw $d_{\mathrm{raw}}=C \times H \times W$ input dimensions to $E_{\psi}: \mathbb{R}^{C\times H\times W} \rightarrow \mathbb{R}^{d_q}$, where $d_q$ denotes the representation supplied to the quantum encoding stage and $C$, $H$, $W$ represent the image channel, height and width. This separation allows raw input complexity and quantum representation dimensionality to be studied independently.

\subsection{Classical Feature and Quantum Encoding}
\label{sec:classical_encoding}
The encoder $E_{\psi}$ can be identity block, a linear block or multi-layer perceptron (MLP), or consist of configurable convolutional blocks for spatial feature extraction followed by a fully connected projection to the required quantum input dimension. Its convolutional channels, pooling, activation, normalization, and dropout are configurable. EXAQC supports rotation, U3, and amplitude input encodings \cite{rath2024quantum}. For single-angle $R_x$, or $R_y$ encoding, each of $n_q$ input qubits receives one feature from $d_q$ input features and therefore $n_q=d_q$. For U3 encoding, three features parameterize each input qubit using

\begin{equation}
U_3(\theta,\phi,\lambda)
=
\begin{bmatrix}
\cos(\theta/2) &
-e^{i\lambda}\sin(\theta/2) \\
e^{i\phi}\sin(\theta/2) &
e^{i(\phi+\lambda)}\cos(\theta/2)
\end{bmatrix},
\end{equation}

giving $n_q=d_q/3$. U3 encoding therefore increases the number of classical features presented to a fixed number of input qubits. Amplitude encoding provides an alternative mapping between classical representations and quantum states where $n_q = \lceil log_2(d_q) \rceil$. 

\subsection{Evolutionary Quantum Circuit Search}
\label{sec:evolutionary_search}

EXAQC represents each PQC as a variable-length circuit genome containing its input and output qubits and an ordered collection of quantum gates. Each gate specifies its operation, associated qubits, circuit position, parameters, innovation number, and enabled state. An architecture is therefore $\mathcal{A}=\{g_1,g_2,\ldots,g_m\}$, where both $m$ and the arrangement of gates are determined through evolution. Candidate circuits can consequently differ in gate composition, depth, parameter count, and qubit connectivity rather than being constrained to a
fixed ans\"atz.

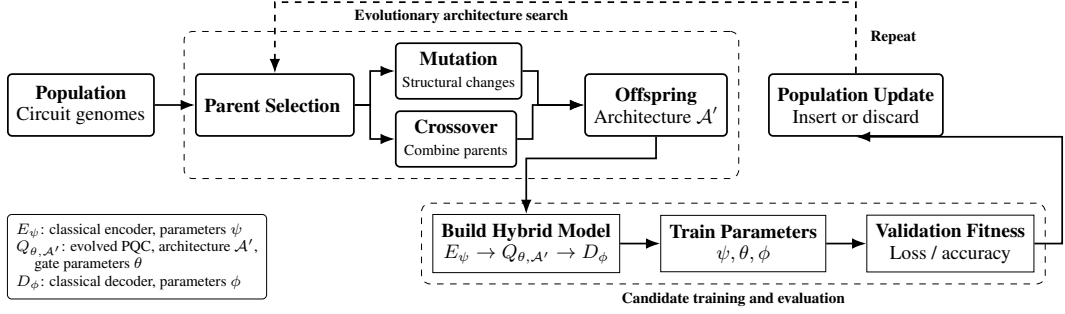
\begin{figure*}[!ht]
    \centering
    \resizebox{\textwidth}{!}{%
    \begin{tikzpicture}[
        font=\small,
        >=Latex,
        node distance=0.55cm and 0.65cm,
        box/.style={
            rectangle,
            rounded corners=2pt,
            draw,
            thick,
            align=center,
            minimum height=1.0cm,
            minimum width=2.15cm,
            inner sep=4pt
        },
        smallbox/.style={
            rectangle,
            rounded corners=2pt,
            draw,
            thick,
            align=center,
            minimum height=0.9cm,
            minimum width=1.95cm,
            inner sep=3pt
        },
        evalbox/.style={
            rectangle,
            draw,
            align=center,
            minimum height=0.9cm,
            minimum width=2.1cm,
            inner sep=4pt
        },
        arrow/.style={
            ->,
            thick,
            >=Latex
        },
        feedback/.style={
            ->,
            thick,
            dashed,
            >=Latex
        }
    ]


    \node[box] (population) {
        \textbf{Population}\\
        Circuit genomes
    };

    \node[box, right=of population] (selection) {
        \textbf{Parent Selection}
    };

    \node[smallbox, right=of selection, yshift=0.55cm] (mutation) {
        \textbf{Mutation}\\
        \scriptsize Structural changes
    };

    \node[smallbox, right=of selection, yshift=-0.55cm] (crossover) {
        \textbf{Crossover}\\
        \scriptsize Combine parents
    };

    \node[box, right=1.0cm of mutation, yshift=-0.55cm] (offspring) {
        \textbf{Offspring}\\
        Architecture $\mathcal{A}'$
    };

    \node[box, right=of offspring] (update) {
        \textbf{Population Update}\\
        Insert or discard
    };


    \draw[arrow] (population.east) -- (selection.west);

    \draw[arrow]
        (selection.east) -- ++(0.25,0)
        |- (mutation.west);

    \draw[arrow]
        (selection.east) -- ++(0.25,0)
        |- (crossover.west);

    \draw[arrow]
        (mutation.east) -- ++(0.25,0)
        |- (offspring.west);

    \draw[arrow]
        (crossover.east) -- ++(0.25,0)
        |- (offspring.west);


    \node[evalbox, below=1.25cm of offspring, xshift=-2.1cm] (build) {
        \textbf{Build Hybrid Model}\\
        $E_{\psi}\rightarrow
        Q_{\theta,\mathcal{A}'}
        \rightarrow D_{\phi}$
    };

    \node[evalbox, right=of build] (train) {
        \textbf{Train Parameters}\\
        $\psi,\theta,\phi$
    };

    \node[evalbox, right=of train] (fitness) {
        \textbf{Validation Fitness}\\
        Loss / accuracy
    };

    \draw[arrow]
        (offspring.south)
        -- ++(0,-0.35)
        -| (build.north);

    \draw[arrow] (build.east) -- (train.west);
    \draw[arrow] (train.east) -- (fitness.west);

    \draw[arrow]
        (fitness.east)
        -- ++(0.45,0)
        |- (update.south);

    
    \draw[feedback]
        (update.north)
        -- ++(0,1.15)
        node[midway, right=3pt, font=\scriptsize] {\textbf{Repeat}}
        -| (selection.north);
        

    \node[
        draw,
        dashed,
        rounded corners,
        inner sep=5pt,
        fit=(selection)(mutation)(crossover)(offspring),
        label={[font=\scriptsize]above:
        \textbf{Evolutionary architecture search}}
    ] (searchgroup) {};

    \node[
        draw,
        dashed,
        rounded corners,
        inner sep=5pt,
        fit=(build)(train)(fitness),
        label={[font=\scriptsize]below:
        \textbf{Candidate training and evaluation}}
    ] {};


    \node[
        draw,
        rounded corners=2pt,
        align=left,
        font=\scriptsize,
        inner sep=4pt,
        anchor=north west
    ] at (population.south west |- build.north) {
        $E_{\psi}$: classical encoder, parameters $\psi$\\
        $Q_{\theta,\mathcal{A}'}$: evolved PQC, architecture $\mathcal{A}'$,\\
        \hspace{1.2em}gate parameters $\theta$\\
        $D_{\phi}$: classical decoder, parameters $\phi$
    };

    \end{tikzpicture}%
    }

    \caption{
        EXAQC evolutionary search workflow for hybrid image classification.
        Parent circuit genomes are selected and modified through mutation or
        crossover to create offspring architectures. Each offspring is
        instantiated within the hybrid model and its classical and quantum
        parameters are optimized using gradient-based training. Validation
        performance determines whether the offspring is inserted into the
        steady-state population. 
    }
    \label{fig:exaqc_evolution}
\end{figure*}

Fig.~\ref{fig:exaqc_evolution} demonstrates how evolution modifies the architecture $\mathcal{A}$ while gradient-based learning optimizes the continuous model parameters. Structural mutation can add, enable, disable, or reorder gates and modify their associated qubits, while binary, exponential, and $n$-ary crossover combine information from parent circuits. Compatible learned parameters are transferred to offspring through Lamarckian parameter inheritance (\cite{lyu2021experimental,kar2026investigating}). Candidate circuits are evaluated in parallel using EXAQC's asynchronous master-worker framework. Under steady-state evolution, evaluated offspring immediately compete for population insertion rather than waiting for generational replacement. EXAQC also supports island-based steady-state evolution, in which multiple subpopulations explore different regions of the architecture space while allowing inter-island information exchange.

\subsection{Training and Quantum Readout}
\label{sec:quantum_training}

For each candidate architecture, the encoder parameters $\psi$, quantum parameters $\theta$, and decoder parameters $\phi$, denoted jointly as $\Theta=\{\psi,\theta,\phi\}$, are optimized by minimizing the cross-entropy loss $\mathcal{L}_{CE} = -\frac{1}{N} \sum_{i=1}^{N} \log p_{\Theta}(y_i \mid \mathbf{x}_i),$ where $N$ is the number of training examples, $\mathbf{x}_i$ is the $i$-th input image, $y_i$ is its class label, and $p_{\Theta}(y_i \mid \mathbf{x}_i)$ is the probability the hybrid model assigns to the correct class, obtained by applying a softmax to the decoder's output logits. Validation performance provides the fitness used by evolutionary selection. After the evolved circuit is applied, designated output qubits are measured and supplied to the decoder. For probability readout over $n_o$ output qubits, $\mathbf{q} = [p_0,p_1,\ldots,p_{2^{n_o}-1}]$, where $p_k$ is the probability of measuring computational basis state $k$, whereas expectation-value readout produces $\mathbf{q} = [\langle Z_1\rangle,\ldots,\langle Z_{n_o}\rangle]$, where $\langle Z_j\rangle$ is the expectation value of the Pauli-$Z$ observable on output qubit $j$. The decoder $D_{\phi}$ then maps $\mathbf{q}$ to the final class logits.
\section{Experimental Design}
\label{sec:experimental_design}

As evolutionary search is stochastic, we report test performance across circuit evaluations and variability across 10 independent runs, indicating how quickly and reliably good circuits are found. Circuit complexity is measured by gate count, and trainable quantum parameters, which determine practicality on near-term hardware and match the resources reported by prior methods. 
All evolutionary experiments use the same EXAQC implementation and quantum simulation backend. Configurations specify population strategy and size, evaluation budget, mutation and crossover settings, optimizer, learning rate, training steps, batch size, encoding and readout modes, and random seeds. Because EXAQC uses asynchronous steady-state evolution, search progress is reported by the number of evaluated circuits rather than conventional generations. Aaggregate statistics are reported across runs. Representative evolved circuits are selected using test performance.






\paragraph{Baselines and Model Configurations}
EXAQC hybrid models are compared against fixed-architecture classical baselines. Classical models use the same image-processing pipeline, replacing the evolved PQC with classical fully connected transformations. Where possible, encoder and decoder configurations are held constant across corresponding models.

\paragraph{Scalability and Representation Studies}
To address \textbf{RQ1}, EXAQC is evaluated on MNIST~\cite{deng2012mnist}, Fashion-MNIST~\cite{xiao2017fashion}, and CIFAR-10~\cite{krizhevsky2009learning} , comparing accuracy, circuit resources, and convergence against prior quantum approaches. To address \textbf{RQ2}, we vary the classical encoder that produces the quantum circuit's input (identity, MLP, or CNN) while holding the evolutionary configuration fixed, isolating how the supplied representation affects accuracy and circuit complexity.
Accuracy, convergence, circuit depth, gate and parameter counts, and computational cost are measured as $d_q$ increases. Representation dimensionality and qubit count are reported separately when the selected encoding maps multiple features to each qubit. This controlled experiment distinguishes scalability with respect to the representation processed by EXAQC from the raw dimensionality of the source image.


\paragraph{Encoding and Parameter Efficiency}
To address \textbf{RQ3}, we compare amplitude, RX, RY, and U3 input encodings on each dataset, holding the encoder and evolutionary configuration fixed, and report test accuracy and evolved gate and parameter counts to assess whether encoding effects are dataset-dependent. To address \textbf{RQ4}, EXAQC-based hybrid models are compared with classical MLP, CNN, and ResNet baselines. Given a classical model with $P_C$ trainable parameters and a hybrid model with $P_H$, parameter reduction is $R_P = \frac{P_C-P_H}{P_C}\times100\%$. Because fewer parameters are only meaningful at comparable accuracy, parameter reduction is always assessed jointly with predictive performance. 


\section{Results}
\label{sec:results}

We evaluate EXAQC along three dimensions: predictive performance and circuit resources relative to prior quantum approaches (RQ1 and RQ2), sensitivity to quantum
encoding (RQ3), and the accuracy parameter trade-off relative to classical architectures (RQ4).

\paragraph{Hyperparameters} EXAQC used a population of 30 and a search budget of 800 candidate genomes for MNIST/Fashion-MNIST and 1000 genomes for CIFAR-10. Candidate parameters were trained for 100 epochs using Adam with a learning rate of $0.001$ and cross-entropy loss. Mutation and Crossover operations were performed adopting threshold values from the classification runs in  ~\cite{kar2026evolutionary}. Experiments were repeated across 10 independent runs. We used 5-6 qubits for MNIST and Fashion-MNIST and 8-10 qubits for CIFAR-10. Quantum input encoding involved amplitude, angle encoding using \textit{RX} and \textit{RY} gates and, U3 encoding where each gate parameter is fed by an encoder output. Dimensionality representation was configured using encoder architectures involving identity, MLP and CNN based feature extractors. The results reported in Tables~\ref{tab:image_prior_work} utilize 1-layered CNN encoders for MNIST and Fashion-MNIST; and a 3-layered CNN encoder for CIFAR-10 to generate best possible results uisng fine-grained features. Table ~\ref{tab:encoding_comparison} utilizes linear encoders (MLPs) for MNIST and Fashion-MNIST to drive circuit exploration based on raw features rather than overly depending on the encoder ; CIFAR-10 utilizes convolutional layers to extract finer features since simple encoders were not able to so efficiently. 50K+10K images were used for training and validation, for MNIST and Fashion MNIST with 10K images for testing. 40K+10K images were used for training and validation with CIFAR-10, and 10K images for testing.







\subsection{Comparison with Prior Image-Classification Approaches}

\begin{table}[!ht]
\centering
\caption{Comparison with prior quantum and quantum-architecture-search approaches on image classification: test accuracy (\%) and quantum circuit resources (G/P = gates / quantum gate parameters). Cls.- specifies how many classes were actually used from the dataset. Accuracies are not necessarily directly comparable, as prior studies differ in number of classes, dataset subsets, preprocessing, and quantum input representations. ``--'' indicates a result or resource count that was not reported.}
\label{tab:image_prior_work}

\setlength{\tabcolsep}{2.5pt}
\renewcommand{\arraystretch}{1.1}
\newcommand{\mth}[2]{\begin{tabular}[c]{@{}l@{}}#1\\[-2pt]{\scriptsize\citep{#2}}\end{tabular}}

\begin{tabular}{@{}l c cc cc cc@{}}
\toprule
& & \multicolumn{2}{c}{\textbf{MNIST}}
  & \multicolumn{2}{c}{\textbf{Fashion-MNIST}}
  & \multicolumn{2}{c}{\textbf{CIFAR-10}} \\
\cmidrule(lr){3-4}\cmidrule(lr){5-6}\cmidrule(l){7-8}
\textbf{Method} & \textbf{Cls.} &
\textbf{Acc.} & \textbf{G/P} &
\textbf{Acc.} & \textbf{G/P} &
\textbf{Acc.} & \textbf{G/P} \\
\midrule
\mth{Hybrid DQAS$^{a}$}{afane2026differentiable}
  & 2  & $97.97 \pm 0.40$ & -- & $93.86 \pm 1.70$ & -- & $71.82 \pm 1.60$ & -- \\
\addlinespace[4pt]

\mth{DQAS$^{b}$}{zhang2022differentiable}
  & 2  & $96.28 \pm 3.90$ & -- & $93.40 \pm 0.09$ & -- & $71.12 \pm 5.70$ & -- \\
  \addlinespace[4pt]
  
\mth{QuantumDARTS$^{b}$}{wu2023quantumdarts}
  & 2  & $95.66 \pm 3.30$ & -- & $92.20 \pm 0.06$ & -- & $70.14 \pm 1.80$ & -- \\
  \addlinespace[4pt]

\midrule

\mth{EQAS}{zhang2023evolutionary}
  & 2 & $94.00$ & 22/\textbf{9} & $75.00$ & {\bf 24/20} & -- & -- \\
  \addlinespace[4pt]
  
\mth{EQNAS$^{e}$}{li2023eqnas}
  & 2 & $87.35$ & 43/20 & $84.15$ & 41/18 & $59.65$ & 43/20 \\
  \addlinespace[4pt]
  
\mth{AQEA-QAS}{li2025aqea}
  & 2 & $91.92$ & 39/16 & $87.15$ & 39/16 & $61.65$ & 39/\textbf{16} \\
  \addlinespace[4pt]

\midrule
  
\mth{Quantum K-NN$^{c,d}$}{maldonado2024quantum}
  & 2  & $94.00 \pm 2.11$ & -- & -- & -- & -- & -- \\
  \addlinespace[4pt]
  
\mth{HQNN-Parallel$^{d}$}{senokosov2024quantum}
  & 10 & $\mathbf{99.21}$ & -- & -- & -- & $82.64$ & -- \\
  \addlinespace[4pt]
\midrule
  
\mth{QNAS$^{f}$}{maleki2026qnas}
  & 10 & $97.16$ & 68/-- & $87.38$ & 42/-- & -- & -- \\
  \addlinespace[4pt]
  
\midrule
\textbf{EXAQC (ours)}
  & 10 & $\mathbf{98.42 \pm 1.03}$ & \textbf{20}/14
       & $\mathbf{90.62 \pm 0.8}$ & 30/21
       & $\mathbf{85.47 \pm 1.37}$ & \textbf{30}/21 \\
\bottomrule
\end{tabular}

\vspace{3pt}
\begin{minipage}{\linewidth}
\footnotesize
$^{a}$Images resized for 4-, 9-, and 16-qubit configurations.
$^{b}$Results from the comparative evaluation in Hybrid DQAS.
$^{c}$MNIST digits 0 and 1 only.
$^{d}$Not an architecture-search method.
$^{e}$Cross-dataset results from the comparative evaluation in AQEA-QAS.
$^{f}$8 qubits/2 layers (MNIST), 5 qubits/2 layers (Fashion-MNIST); gate count follows the original work's transpilation convention.
\end{minipage}
\end{table}

Table~\ref{tab:image_prior_work} compares EXAQC with prior quantum image-classification and architecture-search approaches. EXAQC achieves $98.42 \pm 1.03$\%, $90.62 \pm 0.8$\%, and $85.47 \pm 1.37$\% accuracy on the full 10-class MNIST, Fashion-MNIST, and CIFAR-10 tasks, respectively. On CIFAR-10, this exceeds the reported results of the architecture-search methods included in the comparison, as some approaches use simplified versions of the datasets that are either downscaled in size and/or reduced to fewer classes. 
Table ~\ref{tab:image_prior_work} further compares quantum circuit resources. EXAQC uses 20 gates and 14 parameters on MNIST and 30 gates and 21 parameters on Fashion-MNIST and CIFAR-10. Its evolved circuits use fewer gates than EQNAS and AQEA-QAS architectures, who follow evolutionary search strategies, across all three datasets, although parameter counts are lower only on MNIST.

\subsection{Effect of Quantum Encoding Strategy}
\begin{table*}[!ht]
\centering
\caption{Comparison of quantum encoding strategies for EXAQC across the
image-classification benchmarks. Results are reported over multiple
independent runs as mean $\pm$ standard deviation.}
\label{tab:encoding_comparison}

\small
\setlength{\tabcolsep}{5pt}
\renewcommand{\arraystretch}{1.15}

\begin{tabular}{llccc}
\toprule
\textbf{Dataset}
& \textbf{Quantum Encoding}
& \textbf{Test Acc. (\%)}
& \textbf{\# Gates}
& \textbf{\# Params.} \\
\midrule

\multirow{4}{*}{MNIST}
& Amplitude
& 95.65 $\pm$ 0.1 & 8 & 5  \\
& RX
& 96.22 $\pm$ 0.1 & 11 & 6 \\
& RY
& 96.27 $\pm$ 0.2 & 17 & 9 \\
& U3
& \textbf{96.42 $\pm$ 0.2} & 14 & 14 \\

\midrule

\multirow{4}{*}{Fashion-MNIST}
& Amplitude
& 82.19 $\pm$ 0.2 & 5 & 6 \\
& RX
& \textbf{88.01 $\pm$ 0.8} & 12 & 6 \\
& RY
& 81.90 $\pm$ 1.2 & 13 & 8 \\
& U3
& 87.26 $\pm$ 0.5 & 4 & 4 \\

\midrule

\multirow{4}{*}{CIFAR-10}
& Amplitude
& 61.38 $\pm$ 0.8 & 42 & 22 \\
& RX
& 83.75 $\pm$ 1.4 & 43 & 26 \\
& RY
& 84.37 $\pm$ 1.3 & 45 & 27 \\
& U3
& \textbf{86.68 $\pm$ 1.5} & 60 & 40 \\

\bottomrule
\end{tabular}

\end{table*}

Table~\ref{tab:encoding_comparison} shows that encoding choice has a dataset-dependent effect on both accuracy and circuit complexity. Differences are small on MNIST, where accuracy ranges from $95.65 \pm 0.10$\% to
$96.42 \pm 0.20$\%, but become more pronounced on Fashion-MNIST and CIFAR-10. On CIFAR-10, amplitude encoding achieves $61.38 \pm 0.80$\%, compared with $83.75 \pm 1.40$\%, $84.37 \pm 1.30$\%, and $86.68 \pm 1.50$\% for RX, RY, and U3, respectively. CIFAR-10 also produces larger evolved circuits, containing 42-60 gates and 22-40 parameters, which is expected due to the additional complexity of the task. Because U3 encodes three features per qubit whereas RX and RY encode one, the encodings cannot be compared at both a fixed qubit count and a fixed input dimension. We therefore hold the number of qubits fixed, so U3 receives three times as many input features; matching the input dimension instead would require tripling the qubit count for RX and RY, making the process more computationally expensive.

\subsection{Parameter Reduction}
\begin{table*}[!ht]
\centering
\caption{Comparison of classical and EXAQC-based hybrid models in terms of classification performance and parameter efficiency. Parameter reduction is calculated relative to the corresponding classical baseline for each dataset. Best Results are reported over 10 independent runs using RY angle encoding.}
\label{tab:classical_exaqc_comparison}

\small
\setlength{\tabcolsep}{5pt}
\renewcommand{\arraystretch}{1.15}
\resizebox{\textwidth}{!}{
\begin{tabular}{llccc}
\toprule
\textbf{Dataset}
& \textbf{Model}
& \textbf{Test Acc. (\%)}
& \textbf{\# Trainable Params.}
& \textbf{Param. Reduction ($\mathbf{R_P} \%$)} \\
\midrule

\multirow{2}{*}{MNIST}
& \makecell[l]{MLP\\ $(784 \times 512 \times 128 \times 10)$}
& 97.48 & 468874 & 78.36 ($\downarrow$) \\
& \textbf{EXAQC + MLP}
& 96.83 & 101458 & -- \\

\midrule

\multirow{4}{*}{Fashion-MNIST}
& \makecell[l]{MLP\\ $(784 \times 512 \times 128 \times 10)$}
& 88.03 & 468874 & 76.67 ($\downarrow$) \\
& \textbf{EXAQC + MLP}
& 86.95 & 109421 & -- \\
\cmidrule{2-5}
& 1-layer CNN
& 89.47 & 74922 & 54.83 ($\downarrow$) \\
& 2-layer CNN
& 90.87 & 79418 & 57.38 ($\downarrow$) \\
& \textbf{EXAQC + CNN}
& 91.11 & 33844 & -- \\

\midrule

\multirow{4}{*}{CIFAR-10}
& 3-layer CNN
& 81.11 & 258986 & 61.32 ($\downarrow$) \\
& 10-layer CNN
& 85.50 & 2569962 & 96.10 ($\downarrow$) \\
& ResNet18
& 87.08 & 11173962 & 99.10 ($\downarrow$) \\
& \textbf{\makecell[l]{EXAQC + CNN}}
& 85.68 & 100173 & -- \\

\bottomrule
\end{tabular}
}
\end{table*}

Table~\ref{tab:classical_exaqc_comparison} compares the accuracy and trainable parameter counts of EXAQC-based hybrid models with classical architectures. On MNIST, EXAQC achieves 96.13\% accuracy with 101,458 parameters, compared with 97.48\% using 468,874 parameters for the MLP. On Fashion-MNIST, the CNN-based EXAQC model achieves 91.11\% with 33,844 parameters, compared with 89.47\% and 74,922 parameters for the one-layer CNN. On CIFAR-10, EXAQC achieves 85.68\% with 100,173 parameters, compared with 85.50\% and 2,569,962 parameters for the 10-layer CNN. These results show substantial reductions in trainable parameters while retaining comparable classification performance. 

\subsection{Circuit Visualization}
Figure~\ref{fig:cifar10_genome} shows a representative evolved hybrid architecture for CIFAR-10, the most complex of the three. The encoder applies three convolutional blocks ($3\!\rightarrow\!32\!\rightarrow\!64\!\rightarrow\!128$ feature maps, each followed by $2\times2$ max pooling), and a linear layer then compresses the flattened features to only 8 dimensions. This reduces the 3,072-dimensional input by a factor of 384 before quantum processing. The 8 features are RY-encoded onto 8 qubits, one feature per qubit. The evolved circuit is substantially larger and more densely entangled than the MNIST and Fashion-MNIST circuits. It combines single-qubit rotations, many controlled two-qubit gates, and multi-qubit Ising interactions spanning distant qubits in an irregular, non-layered arrangement that differs from the repeated templates of hand-designed ans\"atze. Probability readout over all 8 qubits produces 256 values, which a linear decoder maps to the 10 class logits. The classical encoder thus performs heavy feature extraction and compression, while the evolved circuit supplies a richer quantum transformation of the compact representation than is needed for the simpler datasets. Representations of circuit genomes as evolved for MNIST and Fashion MNIST dataset have been added in Appendix~\ref{subsec:circuit_viz}. 

\begin{figure}[!h]
    \centering
    \includegraphics[width=\linewidth]{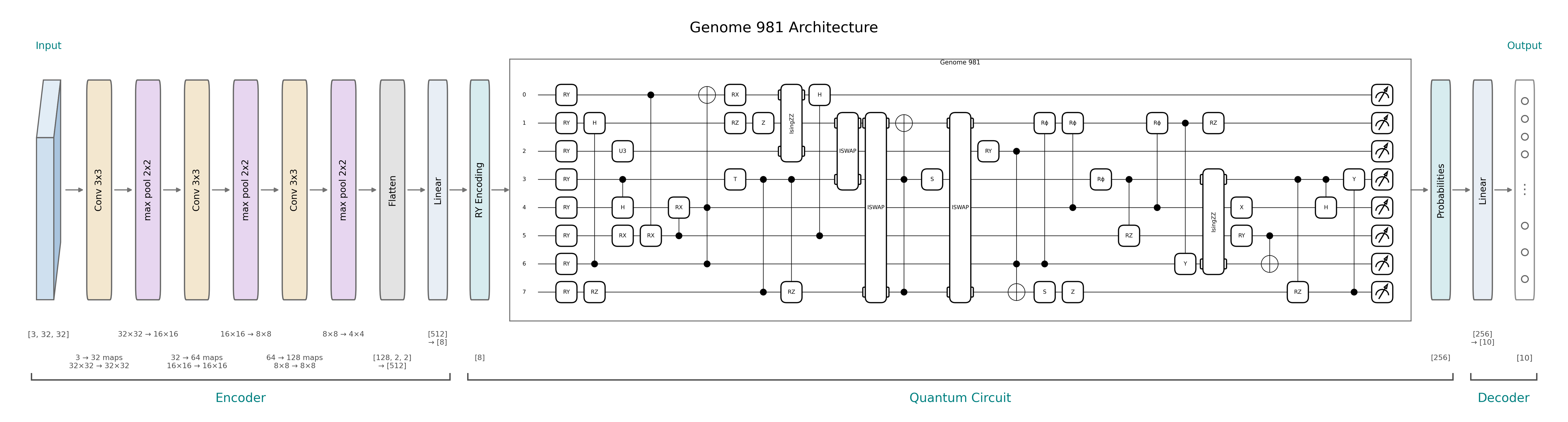}
    \caption{Evolved EXAQC hybrid architecture for CIFAR-10. A three-block CNN encoder compresses each image to 8 features, which are RY-encoded onto 8 qubits.}
    \label{fig:cifar10_genome}
\end{figure}
\section{Discussion}
\label{sec:discussion}

The results demonstrate that EXAQC can be extended to hybrid quantum classical image classification while maintaining compact evolved circuits and competitive predictive performance. Across the three benchmarks,
the experiments reveal differences in performance, circuit complexity, encoding behavior, and parameter requirements as task complexity increases. The convergence and operator analyses further characterize how EXAQC searches and adapts circuit architectures during evolution. We discuss these findings below
in the context of the four research questions.

\textbf{RQ1: \textit{Can EXAQC scale from low-dimensional learning problems to complex image-classification tasks by discovering useful quantum transformations of learned image representations within a hybrid quantum--classical framework?}}
EXAQC remains effective across MNIST, Fashion-MNIST, and CIFAR-10, achieving 98.42\%, 90.62\%, and 85.47\% test accuracy on the full 10-class problems. On CIFAR-10, this exceeds the reported accuracies of the architecture-search methods included in Table~\ref{tab:image_prior_work}, while the evolved circuits use fewer gates than EQNAS and AQEA-QAS on all three datasets. The convergence trajectories in Figure~\ref{fig:loss_evolution} (Refer Appendix \ref{sec:appendixA}) show a consistent pattern across datasets, with rapid initial loss reduction followed by gradual stabilization. Together, these results indicate that evolved PQCs can effectively transform classically extracted image representations as task complexity increases.


\textbf{RQ2: \textit{How does the classical representation supplied to the quantum circuit affect predictive performance and circuit complexity?}}
The choice of classical encoder substantially affects the evolved hybrid models. On Fashion-MNIST, EXAQC with a CNN encoder achieves 91.11\% accuracy with 33,844 parameters, compared with 86.95\% and 109,421 parameters using an MLP encoder (Table~\ref{tab:classical_exaqc_comparison}), so the richer representation improves accuracy while reducing model size. On CIFAR-10, simple encoders did not provide sufficiently informative features, and a CNN encoder was required to reach competitive accuracy. These results suggest that EXAQC performs best when the quantum circuit operates on learned, spatially informed features rather than raw or linearly projected inputs.

\textbf{RQ3: \textit{How does the choice of quantum input encoding affect the accuracy and complexity of evolved circuits, and does this effect depend on the dataset?}}
Encoding effects are strongly dataset-dependent (Table~\ref{tab:encoding_comparison}). On MNIST, all encodings fall within one percentage point (95.65-96.42\%), while on Fashion-MNIST, RX (88.01\%) and U3 (87.26\%) outperform amplitude (82.19\%) and RY (81.90\%) encoding, with U3 achieving this using only 4 gates. On CIFAR-10, the rotation-based encodings (83.75-86.68\%) substantially outperform amplitude encoding (61.38\%), although the best-performing U3 encoding also yields the largest circuits (60 gates, 40 parameters). Encoding choice therefore introduces a dataset-dependent trade-off between accuracy and circuit resources, and its importance grows with task complexity.

\textbf{RQ4: \textit{Do evolved quantum components maintain competitive classification performance relative to larger classical counterparts, despite having fewer trainable parameters?}}
The hybrid models substantially reduce trainable parameters while retaining comparable accuracy (Table~\ref{tab:classical_exaqc_comparison}). On CIFAR-10, EXAQC achieves 85.68\% with 100,173 parameters, matching the 10-layer CNN (85.50\%) with 96.1\% fewer parameters and remaining within 1.4 points of ResNet18 with 99.1\% fewer. On Fashion-MNIST, EXAQC exceeds the one-layer CNN (91.11\% vs.\ 89.47\%) with 54.8\% fewer parameters. On MNIST, the accuracy cost is small (96.83\% vs.\ 97.48\%) for a 78.4\% reduction. The evolved PQCs are also compact at the circuit level, using 20-30 gates. These results demonstrate a favorable accuracy parameter trade-off for the evolved hybrid models.

\section{Limitations and Future Work}
\label{sec:limitations}
While the results demonstrate the potential of EXAQC for compact hybrid image-classification models, they also identify several areas requiring further investigation. First, comparisons with prior work are not strictly controlled because of differences in preprocessing, classification settings, quantum
representations, and resource-counting conventions. Prior works also rely on simplified formulations of the classification datasets. Standardized evaluations would enable more direct comparisons. Second, the effectiveness of the evolved PQC depends on the quality of the features supplied by the classical encoder. High-dimensional image classification therefore requires feature extractors that produce compact quantum-facing representations while preserving task-relevant information.
Future work should investigate joint optimization of the encoder, quantum encoding, and evolved circuit architecture. Third, the current evaluation primarily relies on simulation. Experiments under realistic noise and on quantum hardware are needed to assess the robustness of evolved circuits, with noise-aware evolutionary objectives providing a natural extension. Finally, evolutionary search becomes increasingly expensive as the circuit and representation spaces grow. Future work will explore more efficient distributed search, multi-objective optimization of performance and circuit resources, and evolution of additional hybrid components such as quantum convolutional blocks.
\section{Conclusion}
\label{sec:conclusion}

This work extends EXAQC to evolutionary discovery of quantum circuits for hybrid image classification. Across MNIST, Fashion MNIST, and CIFAR-10, EXAQC evolves compact quantum components that achieve competitive classification performance while substantially reducing the trainable parameter requirements of the surrounding hybrid models. EXAQC achieves $98.42\%$, $90.62\%$, and $85.47\%$ accuracy on the three full 10-class benchmarks, respectively, while the evolved circuits contain only 20-30 gates in the resource comparison. Unlike most prior quantum architecture search methods, which rely on fixed input encodings, EXAQC trains the classical encoder jointly with each evolved circuit, learning the representation supplied to the PQC. This is important as task complexity grows: on CIFAR-10, learned representations with rotation-based encodings reach 83.75-86.68\% accuracy, compared with only 61.38\% using amplitude encoding. To our knowledge, this is the first work to combine learned U3 encodings with evolutionary quantum architecture search, with U3 achieving the highest CIFAR-10 accuracy (86.68\%) among the evaluated encodings. In particular, applying quantum circuits to high-dimensional image data requires effective classical feature extractors that produce compact representations while retaining task-relevant information. On CIFAR-10, the hybrid model achieves 85.68\% accuracy with 100,173 trainable parameters, compared with 85.50\% using 2,569,962 parameters for the 10-layer CNN. Overall, the results support evolutionary search as an approach for discovering compact quantum components, while highlighting the importance of jointly considering feature extraction, quantum encoding, and circuit architecture when designing parameter-efficient hybrid models. 
\section*{Acknowledgments}

This work has been supported by the U.S. National Science Foundation under Grant Number \#2225354. Any opinions, findings, and conclusions or recommendations expressed in this material are those of the authors and do not necessarily reflect the views of the National Science Foundation. We also thank RIT's Research Computing team for their assistance and support~\cite{rit-rc-services}.


\section*{Reproducibility Statement}

We have taken several steps to support reproducibility of the reported experiments. The methodology describes the EXAQC evolutionary search procedure, hybrid quantum classical architecture, quantum encoding and readout strategies, and evaluation metrics, while the experimental setup reports the datasets, model configurations, hyperparameters, and number of independent runs used in our evaluation across Sections~\ref{sec:methodology}-~\ref{sec:discussion}. All experiments use publicly available MNIST, Fashion-MNIST, and CIFAR-10 datasets with appropriate citations. Results are reported across multiple independent runs with measures of variation where applicable in Section ~\ref{sec:results} and the scripts are available in the repository. 


\section*{AI Usage}
Generative AI has been used to design some of the code utilized in these experiments, with a focus on generating unit tests and performing comprehensive code review to find and correct bugs and inconsistencies. It has also been used to polish the writing, identify and correct grammatical errors and typos in the text.


\bibliography{refs}
\bibliographystyle{iclr2027_conference}

\clearpage

\appendix

\section{Appendix - Related Work}
\label{sec:related_work}

\paragraph{Hybrid Quantum Classical Image Classification}
Parameterized quantum circuits (PQCs) have been widely investigated for supervised quantum machine learning \cite{havlivcek2019supervised,schuld2020circuit,perezsalinas2020data}. Their effectiveness depends strongly on data encoding and circuit architecture \cite{schuld2021encoding,sim2019expressibility,hubregtsen2021evaluation}, while increasing circuit complexity can introduce optimization difficulties such as barren plateaus \cite{mcclean2018barren,cerezo2021cost}.

Image classification is challenging for near-term quantum models because of the high dimensionality of image data, so hybrid approaches reduce input dimensionality before quantum processing. Wang et al.~\cite{wang2022development} used principal component analysis or amplitude and angle encodings to fit images onto few qubits, and Liu et al.~\cite{liu2021hybrid} replaced classical convolutional filters with parameterized quantum circuits. Both, however, rely on a single hand-designed circuit and do not couple it to a learned classical encoder. Pesah et al.~\cite{pesah2021absence} proved that quantum convolutional neural networks avoid barren plateaus, but for one predefined architecture designed for quantum data. In contrast, EXAQC evolves circuit structure itself, tailoring gate composition, qubit connectivity, and circuit size to the features of a learned classical encoder.

\paragraph{Quantum Architecture Search}
Quantum architecture search (QAS) instead treats circuit structure as an optimization problem \cite{du2022quantum,du2022qas}. Existing approaches include differentiable architecture search \cite{zhang2020dqas,wu2023quantumdarts,li2023sadqas}, predictor-based search \cite{zhang2021predictor}, hierarchical representations \cite{lourens2023hierarchical}, and other adaptive search strategies \cite{huang2024diversity,ma2024continuous}. Evolutionary computation is particularly suitable for QAS because it can directly modify discrete architectural properties while simultaneously supporting continuous parameter optimization. Earlier evolutionary methods focused on quantum circuit synthesis \cite{lukac2003evolutionary,ruican2008genetic,miranda2021synthesis}, while more recent approaches directly evolve PQC architectures \cite{ding2022evolutionary,zhang2023evolutionary,giovagnoli2023qneat,sunkel2023ga4qco}. Multi-objective \cite{ding2023multi} and adaptive or hybrid evolutionary approaches have also been investigated \cite{li2025aqea,kolle2025evaluating,sunkel2025quantum}.

\paragraph{Quantum Architecture Search for Image Classification}
Recent work has increasingly applied quantum architecture search and automated
circuit design to image-classification problems. EQNAS directly applies
evolutionary quantum neural architecture search to MNIST, Fashion-MNIST, and
CIFAR-10, jointly searching quantum circuit structures for classification
\cite{li2023eqnas}. Evolutionary-based quantum architecture search (EQAS)
similarly explores candidate circuit architectures while removing redundant
parameters to reduce circuit complexity \cite{zhang2023evolutionary}.
More recently, AQEA-QAS introduced an adaptive evolutionary strategy for
quantum architecture search \cite{li2025aqea}, while QNAS considers accurate
and efficient quantum neural networks through a hardware-aware,
multi-objective architecture-search formulation \cite{maleki2026qnas}. Differentiable search provides an alternative to evolutionary architecture
optimization. DQAS relaxes discrete circuit-design choices into a
differentiable optimization problem \cite{zhang2022differentiable}, while
QuantumDARTS adapts differentiable architecture search to variational quantum
algorithms \cite{wu2023quantumdarts}. Recent Hybrid DQAS work further applies
differentiable architecture search to quantum computer vision, including
MNIST, Fashion-MNIST, and CIFAR-10 classification
\cite{afane2026differentiable}. Collectively, these methods demonstrate that
automated architecture selection can replace manually specified PQC ans\"atze,
but differ substantially in search strategy, preprocessing, number of classes,
input representation, and circuit-resource constraints. Other quantum image-classification approaches provide relevant performance
comparisons without explicitly performing architecture search. Senokosov
et al.\ investigate hybrid quantum neural networks for image classification
\cite{senokosov2024quantum}, while quantum K-nearest-neighbor methods provide
an alternative quantum classification paradigm \cite{maldonado2024quantum}.
These approaches broaden the comparison between architecture-searched and
fixed or algorithm-specific quantum classifiers, although differences in
experimental settings prevent direct accuracy comparisons.

EXAQC introduced in \cite{kar2026investigating,kar2026evolutionary} performs gate-level structural evolution through mutation and crossover while optimizing the parameters of
the resulting circuits. Candidate architectures can vary in gate composition, connectivity, depth, and parameter count throughout evolution. In contrast to evaluating architecture search only on compressed low-dimensional inputs, this work
examines EXAQC within hybrid image-classification pipelines and distinguishes two forms of scalability: \emph{task-level scalability}, across increasingly complex image-classification problems, and \emph{representation-level
scalability}, defined by the dimensionality presented to the quantum circuit after classical feature extraction. We further evaluate circuit resources and the extent to which evolved quantum components can replace larger classical
components while retaining competitive classification performance.

\section{Appendix}\label{sec:appendixA}

\subsection{Parameter-Efficiency Analysis}
\label{sec:parameter_efficiency_analysis}

We analyze the conditions under which replacing a classical transformation with an evolved quantum classical component reduces the total number of trainable parameters. Importantly, the analysis includes the classical projection and readout layers required to interface with the PQC rather than
considering only the parameters internal to the quantum circuit.

Let a classical transformation map a $d_i$-dimensional input to a $d_o$-dimensional output. The corresponding hybrid replacement first projects the input to a $d_q$-dimensional quantum-facing representation, applies an evolved PQC $\mathcal{A}$, produces a $d_r$-dimensional quantum readout, and
maps this representation to the required $d_o$ outputs.

\paragraph{Classical Transformation Parameter Count}

A fully connected affine transformation$h_C:\mathbb{R}^{d_i}\rightarrow\mathbb{R}^{d_o}
$ contains $P_C=(d_i+1)d_o$ trainable parameters.
The weight matrix of the transformation has dimension
$d_o\times d_i$ and therefore contains $d_i d_o$ trainable parameters.
The bias vector contains an additional $d_o$ parameters. Hence,
\begin{equation}
P_C=d_i d_o+d_o=(d_i+1)d_o.
\label{eq:classical}
\end{equation}

\paragraph{Hybrid Transformation Parameter Count}

Consider the hybrid transformation
\[
\mathbb{R}^{d_i}
\rightarrow
\mathbb{R}^{d_q}
\rightarrow
Q_{\mathcal{A}}
\rightarrow
\mathbb{R}^{d_r}
\rightarrow
\mathbb{R}^{d_o},
\]
where the first and final mappings are fully connected affine
transformations and $Q_{\mathcal{A}}$ is an evolved PQC containing
$P_Q(\mathcal{A})$ trainable parameters. The total number of trainable
parameters in the hybrid transformation is
\[
P_H
=
(d_i+1)d_q
+
P_Q(\mathcal{A})
+
(d_r+1)d_o.
\]

The input projection from $\mathbb{R}^{d_i}$ to
$\mathbb{R}^{d_q}$ contains
\[
d_i d_q+d_q=(d_i+1)d_q
\]
parameters. The evolved PQC contributes $P_Q(\mathcal{A})$ trainable parameters. Finally, the affine transformation from the $d_r$-dimensional quantum
readout to the $d_o$-dimensional output contains
\[
d_r d_o+d_o=(d_r+1)d_o
\]
parameters. Summing these contributions gives
\begin{equation}
P_H
=
(d_i+1)d_q
+
P_Q(\mathcal{A})
+
(d_r+1)d_o.
\label{eq:hybrid}
\end{equation}

\begin{theorem}[Parameter-Efficient Hybrid Replacement]
\label{thm:parameter_efficiency}

A hybrid quantum classical transformation contains fewer trainable
parameters than the classical affine transformation it replaces if and only
if
\[
P_Q(\mathcal{A}) < B_Q,
\]
where
\[
B_Q
=
(d_i+1)d_o
-
(d_i+1)d_q
-
(d_r+1)d_o.
\]

The quantity $B_Q$ defines the maximum trainable-parameter budget available
to the evolved PQC while preserving a strict reduction in the total number
of trainable parameters.
\end{theorem}

\begin{proof}
From Eq~\ref{eq:classical}, the classical transformation
contains
\[
P_C=(d_i+1)d_o
\]
parameters. From Eq~\ref{eq:hybrid}, the hybrid
replacement contains
\[
P_H
=
(d_i+1)d_q
+
P_Q(\mathcal{A})
+
(d_r+1)d_o.
\]

The hybrid transformation is parameter-efficient precisely when
\[
P_H<P_C.
\]

Substituting the expressions for $P_H$ and $P_C$ gives
\[
(d_i+1)d_q
+
P_Q(\mathcal{A})
+
(d_r+1)d_o
<
(d_i+1)d_o.
\]

Rearranging for $P_Q(\mathcal{A})$ yields
\[
P_Q(\mathcal{A})
<
(d_i+1)d_o
-
(d_i+1)d_q
-
(d_r+1)d_o.
\]

Defining
\[
B_Q
=
(d_i+1)d_o
-
(d_i+1)d_q
-
(d_r+1)d_o
\]
therefore gives
\[
P_Q(\mathcal{A})<B_Q.
\]

Since every algebraic step is reversible, the condition is both necessary
and sufficient.
\end{proof}

\begin{corollary}[Feasibility of Parameter-Efficient Replacement]
\label{cor:parameter_efficiency_feasibility}

A necessary condition for the existence of a parameter-efficient hybrid
replacement is
\[
B_Q>0.
\]
If $B_Q\leq0$, no PQC containing a nonnegative number of trainable parameters
can make the proposed hybrid transformation smaller than the classical
transformation it replaces.
\end{corollary}

\begin{proof}
By definition,
\[
P_Q(\mathcal{A})\geq0.
\]

Theorem~\ref{thm:parameter_efficiency} requires
\[
P_Q(\mathcal{A})<B_Q.
\]

If $B_Q\leq0$, this would require
\[
P_Q(\mathcal{A})<0,
\]
which is impossible. Therefore, $B_Q>0$ is necessary for a
parameter-efficient hybrid replacement to exist.
\end{proof}

\begin{corollary}[Effect of Quantum Readout Dimension]
\label{cor:readout_parameter_efficiency}

For fixed $d_i$, $d_o$, and $d_q$, the available quantum parameter budget
$B_Q$ decreases linearly with the quantum readout dimension $d_r$.
Specifically,
\[
\frac{\partial B_Q}{\partial d_r}=-d_o.
\]
Consequently, smaller quantum readout representations provide a larger
parameter budget for the evolved PQC.
\end{corollary}

\begin{proof}
From Theorem~\ref{thm:parameter_efficiency},
\[
B_Q
=
(d_i+1)d_o
-
(d_i+1)d_q
-
(d_r+1)d_o.
\]

Holding $d_i$, $d_o$, and $d_q$ fixed and differentiating with respect to
$d_r$ gives
\[
\frac{\partial B_Q}{\partial d_r}
=
-d_o.
\]

Since $d_o>0$, $B_Q$ strictly decreases as $d_r$ increases.
\end{proof}

\begin{corollary}[Expectation-Value versus Probability Readout]
\label{cor:readout_comparison}

Let $n_o$ denote the number of measured output qubits. If expectation-value
readout produces
\[
d_r^{\mathrm{exp}}=n_o
\]
outputs and full probability readout produces
\[
d_r^{\mathrm{prob}}=2^{n_o}
\]
outputs, then the corresponding PQC parameter budgets satisfy
\[
B_Q^{\mathrm{exp}}
-
B_Q^{\mathrm{prob}}
=
(2^{n_o}-n_o)d_o.
\]

Therefore,
\[
B_Q^{\mathrm{exp}}\geq B_Q^{\mathrm{prob}},
\]
with strict inequality for $n_o\geq2$.
\end{corollary}

\begin{proof}
For expectation-value readout,
\[
B_Q^{\mathrm{exp}}
=
(d_i+1)d_o
-
(d_i+1)d_q
-
(n_o+1)d_o.
\]

For probability readout,
\[
B_Q^{\mathrm{prob}}
=
(d_i+1)d_o
-
(d_i+1)d_q
-
(2^{n_o}+1)d_o.
\]

Subtracting gives
\[
B_Q^{\mathrm{exp}}
-
B_Q^{\mathrm{prob}}
=
(2^{n_o}-n_o)d_o.
\]

Since
\[
2^{n_o}\geq n_o
\]
for $n_o\geq1$, the expectation-value readout provides a quantum parameter
budget at least as large as the probability readout budget. For $n_o\geq2$,
$2^{n_o}>n_o$, giving strict inequality.
\end{proof}

\begin{lemma}[Gate-Level Quantum Parameter Count]
\label{lem:gate_parameter_count}

Let an evolved EXAQC architecture be represented by
\[
\mathcal{A}=\{g_1,\ldots,g_m\},
\]
and let $p(g_j)\geq0$ denote the number of trainable parameters associated
with gate $g_j$. Then
\[
P_Q(\mathcal{A})
=
\sum_{j=1}^{m}p(g_j).
\]

Consequently, the architecture is within the parameter-efficient region if
\[
\sum_{j=1}^{m}p(g_j)<B_Q.
\]
\end{lemma}

\begin{proof}
Each trainable gate contributes $p(g_j)$ independent trainable parameters.
Summing over all enabled trainable gates therefore gives
\[
P_Q(\mathcal{A})
=
\sum_{j=1}^{m}p(g_j).
\]

Substitution into Theorem~\ref{thm:parameter_efficiency} gives
\[
\sum_{j=1}^{m}p(g_j)<B_Q.
\]
\end{proof}

\begin{corollary}[Rotation and U3 Gate Budget]
\label{cor:rotation_u3_budget}

Suppose an evolved circuit contains $m_1$ trainable single-parameter
rotation gates and $m_3$ trainable U3 gates, with all remaining gates
non-parameterized. Then
\[
P_Q(\mathcal{A})=m_1+3m_3,
\]
and the hybrid replacement is parameter-efficient if and only if
\[
m_1+3m_3<B_Q.
\]
\end{corollary}

\begin{proof}
Each single-parameter rotation contributes one trainable parameter and each
U3 gate contributes three. Therefore,
\[
P_Q(\mathcal{A})
=
m_1+3m_3.
\]

Applying Theorem~\ref{thm:parameter_efficiency} gives the stated condition.
\end{proof}

\paragraph{Remark.}
The parameter counts above include only trainable model parameters.
Data-dependent values used for quantum encoding are not counted as trainable
quantum parameters. For example, the three values supplied to a U3 gate as
input features are excluded from $P_Q(\mathcal{A})$ unless they are themselves
optimized model parameters. Similarly, fixed quantum gates contribute zero
trainable parameters even though they contribute to circuit depth and
execution cost.

\subsection{Datasets}

\begin{figure*}[!ht]
    \centering

    \begin{subfigure}[t]{0.32\textwidth}
        \centering
        \includegraphics[width=\linewidth]{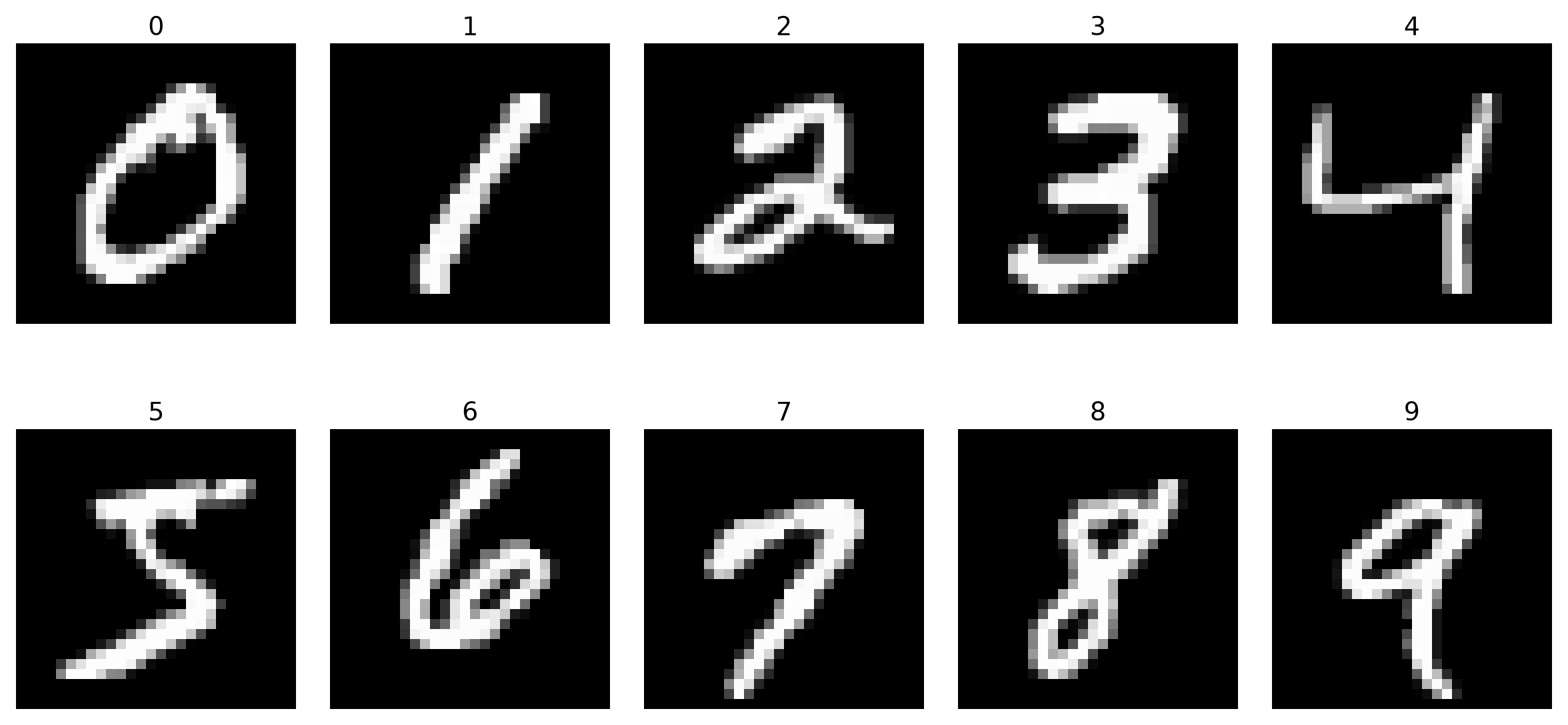}
        \caption{MNIST}
        \label{fig:mnist_samples}
    \end{subfigure}
    \hfill
    \begin{subfigure}[t]{0.32\textwidth}
        \centering
        \includegraphics[width=\linewidth]{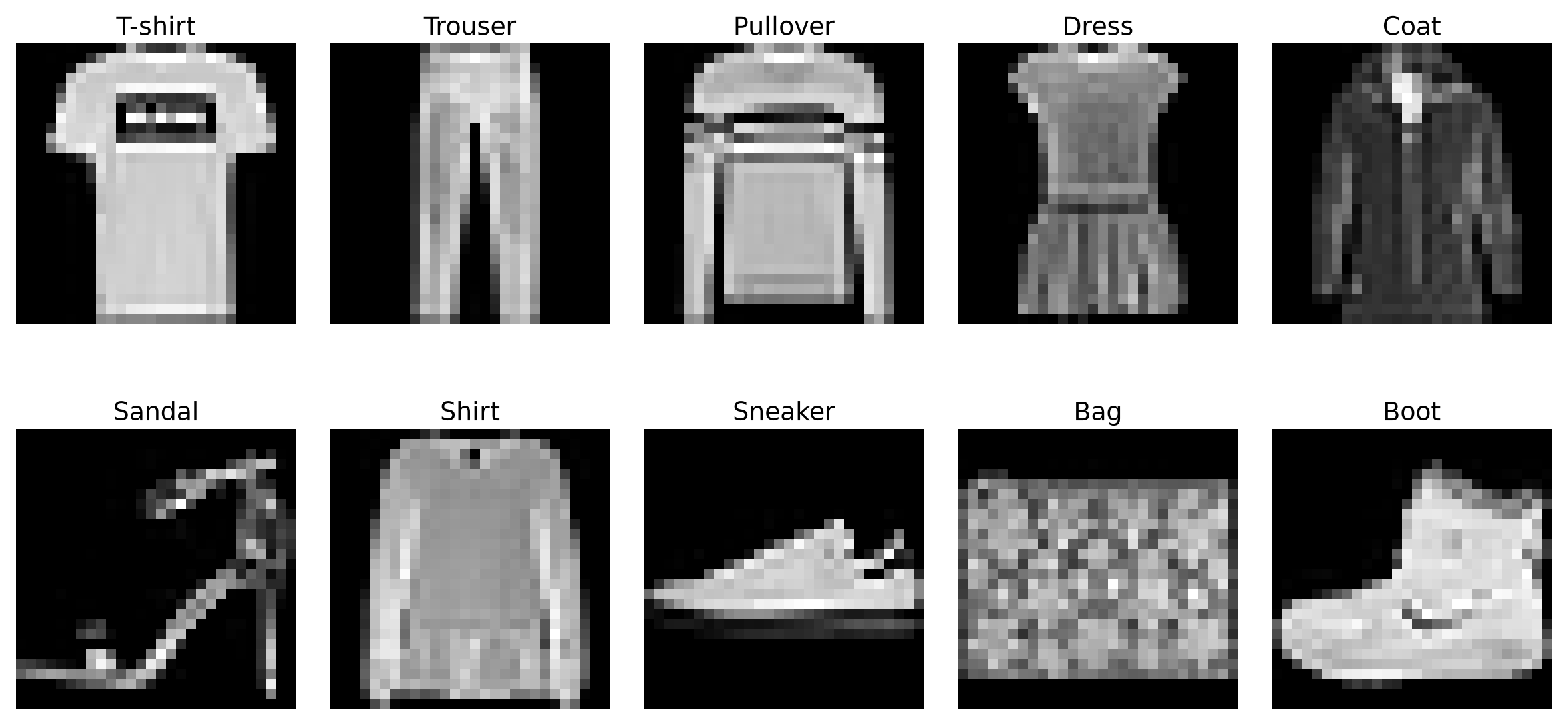}
        \caption{Fashion-MNIST}
        \label{fig:fashion_mnist_samples}
    \end{subfigure}
    \hfill
    \begin{subfigure}[t]{0.32\textwidth}
        \centering
        \includegraphics[width=\linewidth]{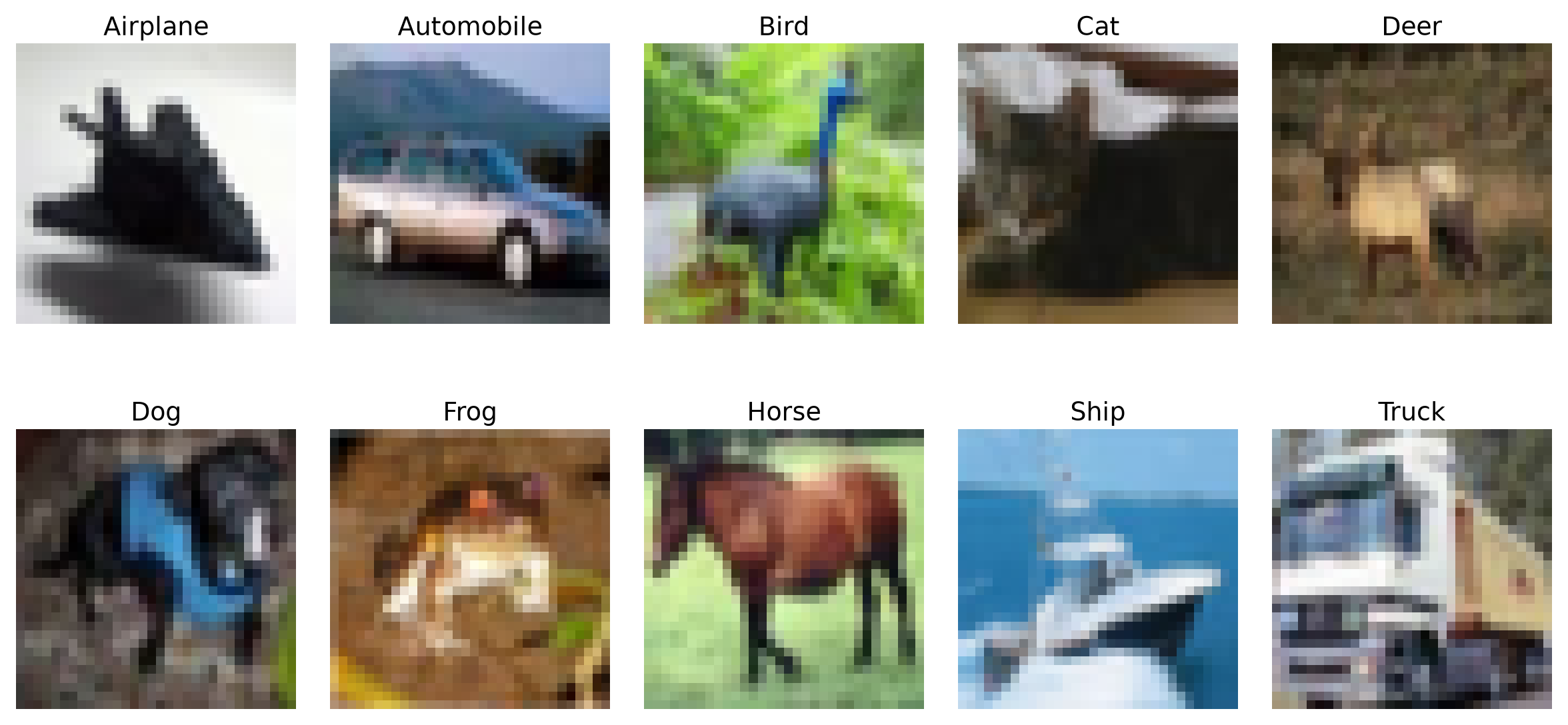}
        \caption{CIFAR-10}
        \label{fig:cifar10_samples}
    \end{subfigure}

    \caption{Representative samples from the three image-classification
    benchmarks considered in this work. Each panel contains one example
    from each of the ten classes. MNIST and Fashion-MNIST consist of
    $28\times28$ grayscale images, whereas CIFAR-10 consists of
    $32\times32$ RGB images.}
    \label{fig:datasets}
\end{figure*}

\paragraph{MNIST}
The MNIST dataset \cite{deng2012mnist} is a widely used benchmark for handwritten digit classification. It consists of 70,000 grayscale images of handwritten digits from 0 to 9, with each image having a resolution of $28 \times 28$ pixels. The dataset is divided into 60,000 training images and 10,000 test images across 10 classes. Due to its relatively simple visual structure, MNIST provides a standard baseline for evaluating image classification models.

\paragraph{Fashion-MNIST}
Fashion-MNIST \cite{xiao2017fashion} is a grayscale image classification dataset designed as a more challenging alternative to MNIST. It contains 70,000 images of clothing and fashion products, each with a resolution of $28 \times 28$ pixels, divided into 10 classes: T-shirt/top, trouser, pullover, dress, coat, sandal, shirt, sneaker, bag, and ankle boot. The dataset contains 60,000 training images and 10,000 test images. Compared with MNIST, Fashion-MNIST exhibits greater visual similarity between classes, resulting in a more challenging classification task.

\paragraph{CIFAR-10}
CIFAR-10 \cite{krizhevsky2009learning} consists of 60,000 color images with a resolution of $32 \times 32$ pixels and three RGB channels. The images are evenly distributed across 10 classes: airplane, automobile, bird, cat, deer, dog, frog, horse, ship, and truck. The standard split contains 50,000 training images and 10,000 test images, with 6,000 images per class overall. CIFAR-10 presents a more challenging classification problem than MNIST and Fashion-MNIST due to its natural-image content, color information, greater intra-class variability, and more complex visual features.

\subsection{Encoder Strategies}
\label{subsec:encoding_strategies}
A key challenge in hybrid quantum classical learning is mapping classical input data to a representation compatible with the limited number of qubits available to the quantum circuit. EXAQC \footnote{Code is available at \url{https://github.com/travisdesell/exaqc}} separates this process into a classical encoder and a subsequent quantum input encoding. The classical encoder transforms the original feature vector or image into the dimensionality required by the quantum circuit. EXAQC currently supports three encoder architectures: \textit{identity}, \textit{linear}, and \textit{convolutional neural network (CNN)} encoders.

\paragraph{Identity Encoder.}
The identity encoder provides the simplest mapping and directly passes the classical input features to the quantum circuit without introducing additional trainable parameters. Given an input vector $\mathbf{x} \in \mathbb{R}^{d}$, the encoder is defined as
\begin{equation}
E_{\mathrm{id}}(\mathbf{x}) = \mathbf{x}.
\end{equation}
Consequently, the input and output dimensionalities of the encoder must be equal. This strategy is suitable when the input data have already been reduced to the dimensionality required by the quantum circuit and avoids introducing additional classical trainable parameters.

\paragraph{Linear Encoder.}
For inputs whose dimensionality is larger than the number of values required by the quantum circuit, EXAQC provides a trainable linear encoder. Given an input vector $\mathbf{x} \in \mathbb{R}^{d}$, the encoder computes
\begin{equation}
E_{\mathrm{linear}}(\mathbf{x})
= \mathbf{W}\mathbf{x} + \mathbf{b},
\end{equation}
where $\mathbf{W} \in \mathbb{R}^{m \times d}$ and $\mathbf{b} \in \mathbb{R}^{m}$ are trainable parameters and $m$ denotes the required quantum-input dimensionality. For image inputs, the image is first flattened before the linear transformation is applied. This approach therefore learns a compact projection of the original feature space while allowing the projection parameters to be optimized jointly with the hybrid model.

\paragraph{CNN Encoder.}
For image-based tasks, EXAQC additionally supports a configurable CNN encoder that preserves and extracts spatial information before constructing the quantum input representation. Instead of immediately flattening the image, the encoder applies a sequence of configurable convolutional blocks followed by adaptive pooling and a fully connected projection head. In general, the transformation can be expressed as
\begin{equation}
\mathbf{z}
=
E_{\mathrm{CNN}}(\mathbf{X})
=
f_{\mathrm{out}}
\left(
f_{\mathrm{FC}}
\left(
\mathrm{Flatten}
\left(
f_{\mathrm{pool}}
\left(
f_{\mathrm{conv}}(\mathbf{X})
\right)
\right)
\right)
\right),
\end{equation}
where $\mathbf{X}$ is the input image, $f_{\mathrm{conv}}$ represents the configurable convolutional feature extractor, $f_{\mathrm{pool}}$ denotes adaptive pooling, $f_{\mathrm{FC}}$ is the fully connected projection head, and $f_{\mathrm{out}}$ maps the resulting features to the values supplied to the quantum circuit. The convolutional blocks can be configured with different channel sizes, activation functions, and pooling operations, while the projection head can contain multiple fully connected layers and dropout. The final output activation can be configured as \texttt{tanh}, \texttt{sigmoid}, or \texttt{identity}. This encoder is particularly relevant for MNIST, Fashion-MNIST, and CIFAR-10 because it allows spatial feature extraction to be performed classically before the resulting compact representation is supplied to the evolved quantum circuit.

\paragraph{Interface with Quantum Encoding.}
The output of the classical encoder is subsequently mapped onto the input qubits according to the selected quantum input mode. For rotation-based encodings, a single classical value is required per input qubit, whereas the \texttt{u3} mode requires three values per qubit corresponding to the parameters of the general single-qubit rotation. Thus, for $n$ input qubits, the required encoder output dimensionality is
\begin{equation}
d_{\mathrm{enc}} =
\begin{cases}
n & \text{single-parameter rotation encoding}, \\
3n & \text{\texttt{u3} encoding},
\end{cases}
\end{equation}
EXAQC additionally supports amplitude-based quantum inputs, for which an $n$-qubit register can represent at most $2^n$ amplitudes. Accordingly, the framework validates that the encoder output does not exceed the available amplitude-space dimensionality,
\begin{equation}
d_{\mathrm{enc}} \leq 2^n.
\end{equation}
This separation between classical feature extraction and quantum-state preparation enables EXAQC to accommodate both low-dimensional tabular inputs and high-dimensional image data while controlling the dimensionality of the information presented to the evolved quantum circuit.

\subsection{Decoder Strategies}
\label{subsec:decoding_strategies}

Following execution of the evolved quantum circuit, EXAQC employs a decoder to transform the measured quantum outputs into the dimensionality required by the downstream learning task. This provides an interface between the quantum representation and the final task-specific output, such as class logits for classification. EXAQC currently supports two decoding strategies: \textit{clipped} and \textit{linear} decoding. Let the output obtained from the quantum circuit be represented as
\begin{equation}
    \mathbf{z} =
    \left[z_1,z_2,\ldots,z_{d_q}\right]
    \in \mathbb{R}^{d_q},
\end{equation}
where $d_q$ denotes the dimensionality of the quantum circuit output. The decoder maps this representation to
\begin{equation}
    D:\mathbb{R}^{d_q}\rightarrow\mathbb{R}^{d_o},
\end{equation}
where $d_o$ is the output dimensionality required by the learning task.

\paragraph{Clipped Decoder.}
The clipped decoder provides a parameter-free transformation by retaining
only the first $d_o$ elements of the quantum output:
\begin{equation}
    D_{\mathrm{clip}}(\mathbf{z})
    =
    \left[z_1,z_2,\ldots,z_{d_o}\right],
    \qquad d_o \leq d_q.
\end{equation}
Any remaining quantum outputs are discarded. Consequently, this decoder does not introduce additional trainable parameters and requires the quantum circuit to produce at least as many outputs as required by the task. Importantly, the retained values are returned directly without additional normalization. For classification, these values can therefore be interpreted as logits and supplied directly to a cross-entropy loss, which performs the required normalization internally.

\paragraph{Linear Decoder.}
The linear decoder introduces a trainable fully connected layer that maps the quantum output representation to the required output dimensionality. Given quantum outputs $\mathbf{z} \in \mathbb{R}^{d_q}$, the decoder
computes
\begin{equation}
    D_{\mathrm{linear}}(\mathbf{z})
    =
    \mathbf{W}_{d}\mathbf{z}+\mathbf{b}_{d},
\end{equation}
where
$\mathbf{W}_{d}\in\mathbb{R}^{d_o\times d_q}$ and
$\mathbf{b}_{d}\in\mathbb{R}^{d_o}$ are trainable parameters. The decoder
therefore introduces
\begin{equation}
    P_{\mathrm{decoder}}
    =
    d_qd_o+d_o
    =
    d_o(d_q+1)
\end{equation}
additional trainable parameters. Unlike clipped  decoding, the linear decoder can both reduce and expand the dimensionality of the quantum output $\mathbf{z}$ while learning combinations of information obtained from different quantum measurements.

For the image-classification experiments considered in this work, $d_o=10$ for MNIST, Fashion-MNIST, and CIFAR-10. Thus, the decoder ultimately produces ten values corresponding to the ten target classes. The clipped decoder directly selects ten quantum outputs when $d_q\geq10$, whereas the linear decoder learns a mapping from the $d_q$ quantum outputs to the ten class logits $\hat{\mathbf{y}}$:
\begin{equation}
    \mathbf{z}\in\mathbb{R}^{d_q}
    \xrightarrow{D}
    \hat{\mathbf{y}}\in\mathbb{R}^{10}.
\end{equation}

Together with the classical encoder, the decoder allows the evolved quantum circuit to operate as an intermediate representation within a hybrid architecture,
\begin{equation}
    \mathbf{x}
    \xrightarrow{E}
    \mathbf{x}_{q}
    \xrightarrow{Q_{\theta}}
    \mathbf{z}
    \xrightarrow{D}
    \hat{\mathbf{y}},
\end{equation}
where $E$ denotes the classical encoder, $Q_{\theta}$ represents the evolved parameterized quantum circuit, and $D$ denotes the selected decoder. This modular formulation enables EXAQC to vary the dimensionality of the quantum representation independently of the dimensionality required by the final learning task.

\begin{figure}[!t]
	\centering
	\begin{subfigure}{0.32\linewidth}
		\includegraphics[width=\linewidth]{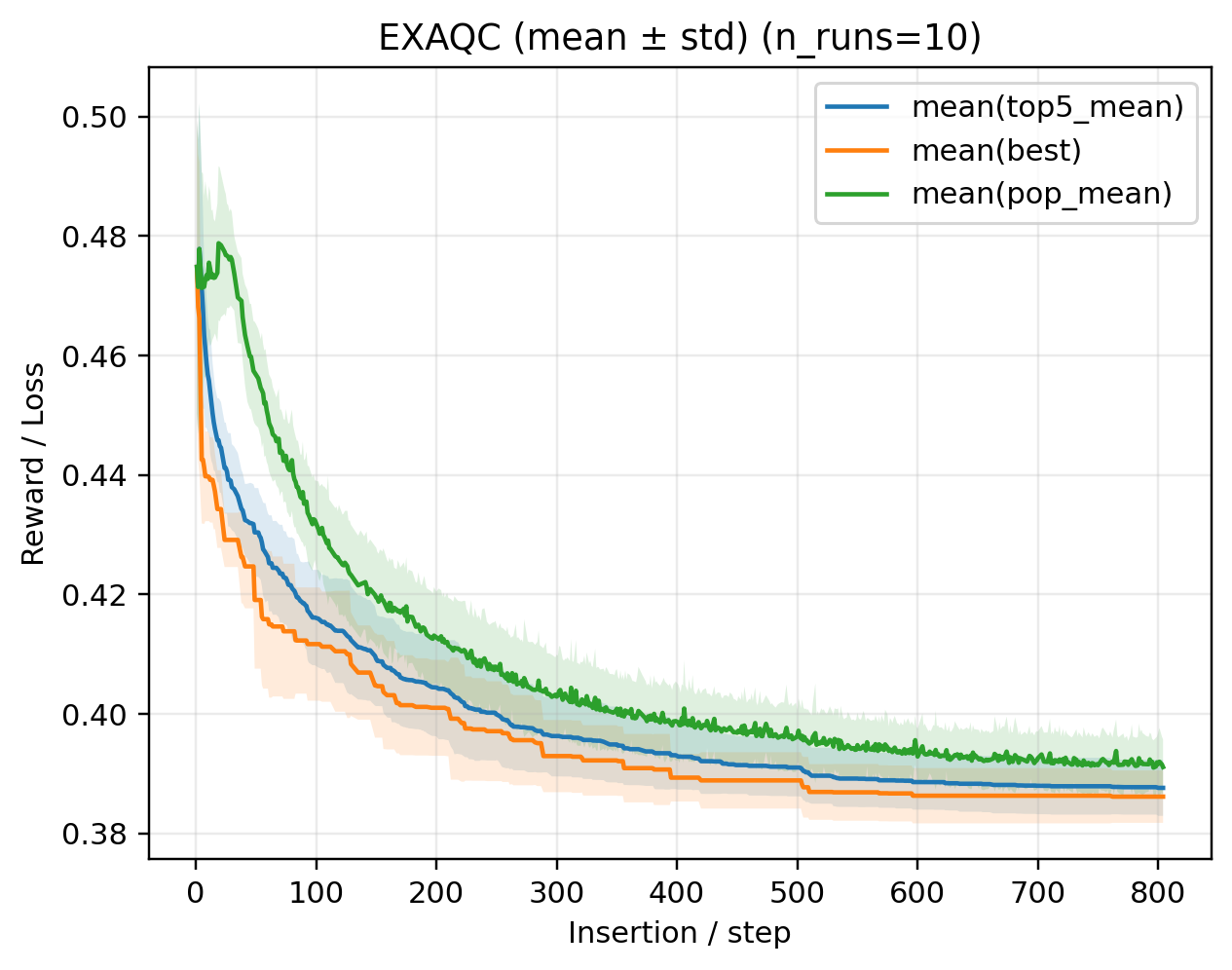}
		\caption{MNIST}
		\label{fig:mnist_loss_plot}
	\end{subfigure}
	\begin{subfigure}{0.32\linewidth}
		\includegraphics[width=\linewidth]{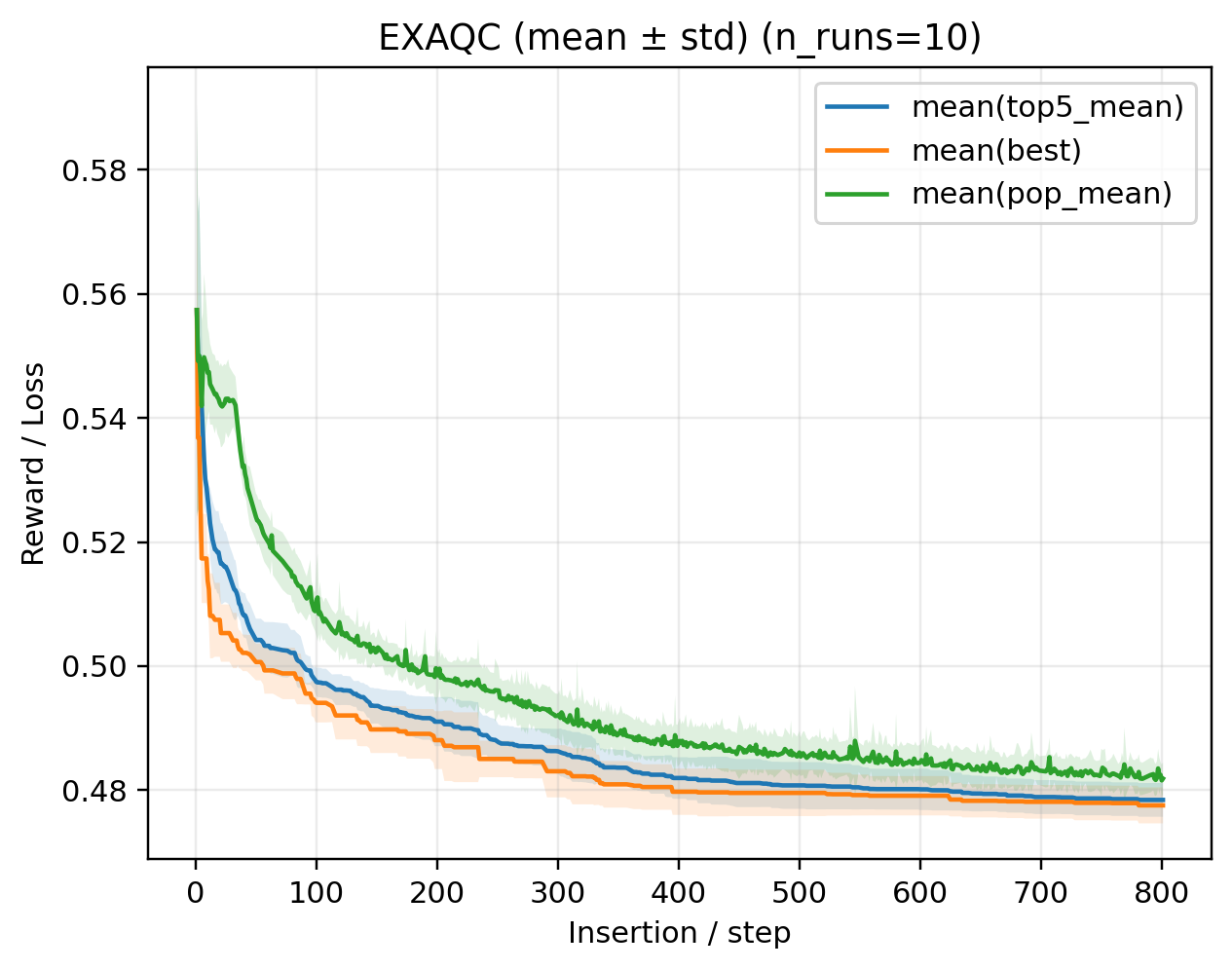}
		\caption{Fashion-MNIST}
		\label{fig:fashion_mnist_loss_plot}
	\end{subfigure}
	\begin{subfigure}{0.34\linewidth}
	        \includegraphics[width=\linewidth]{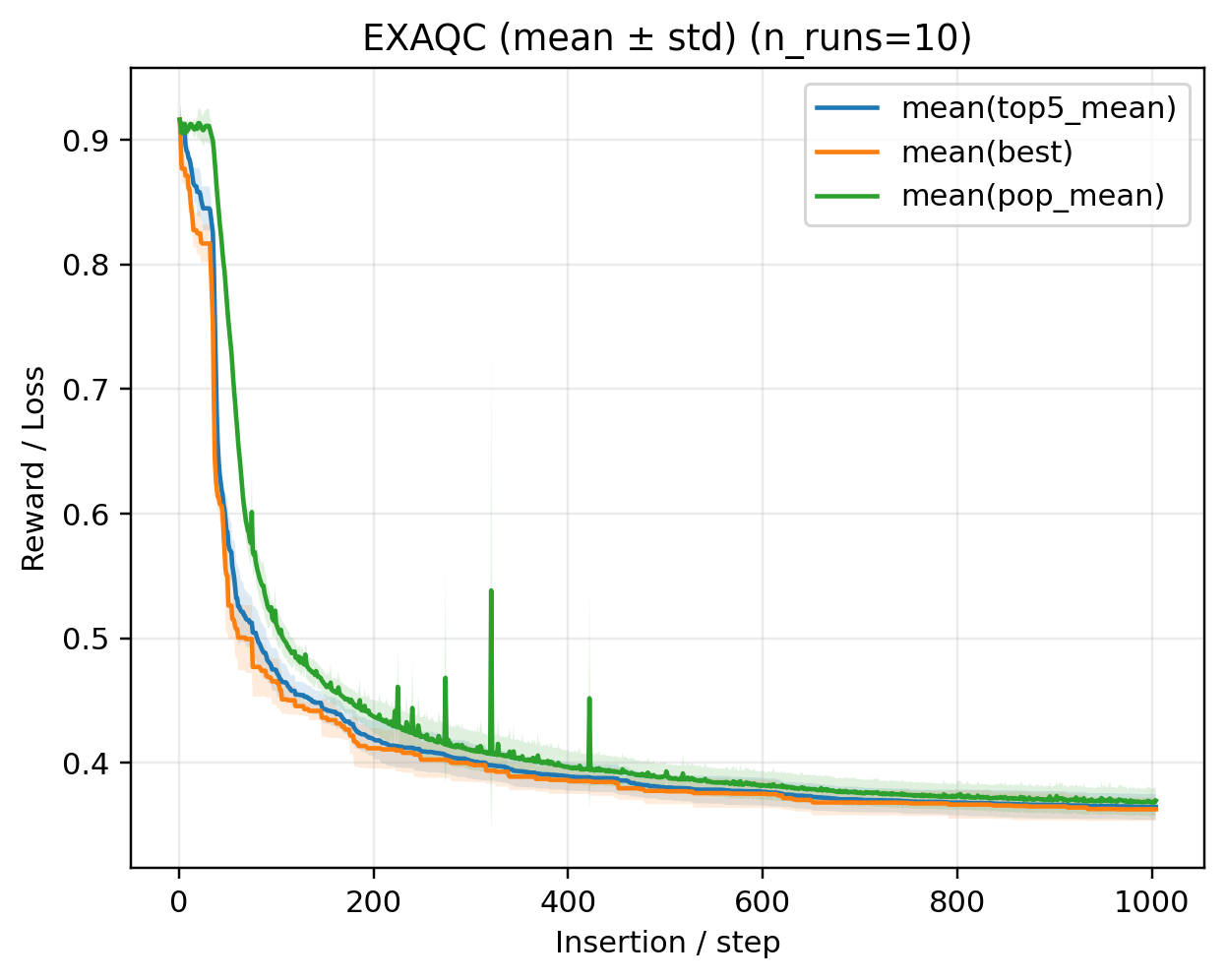}
	        \caption{CIFAR-10}
	        \label{fig:cifar10_loss_plot}
         \end{subfigure}
	\caption{Evolution of classification loss during EXAQC search on (a) MNIST, (b) Fashion-MNIST, and (c) CIFAR-10 across 10 independent runs. Curves show the mean loss for the top-performing individuals, the best individual, and the overall population, with shaded regions indicating $\pm$ one standard deviation across runs.
    }
	\label{fig:loss_evolution}
\end{figure}

\subsection{Operator Analysis}

\begin{table*}[!ht]
\centering

\resizebox{0.62\linewidth}{!}{
\begin{tabular}{p{2.2cm}p{2cm}p{1.8cm}p{2cm}p{1.8cm}}
\toprule
&
&
{\bf MNIST} &
{\bf Fashion-MNIST} &
{\bf CIFAR-10} \\
\midrule

\multirowcell{3}{exponential\\crossover}
& global best
& {\bf 0.027}
& 0.039
& 0.024 \\
& inserted
& 0.227
& 0.336
& 0.287 \\
& discarded
& 0.747
& 0.625
& 0.689 \\
\hline

\multirowcell{3}{reorder\\gate}
& global best
& 0.016
& \textbf{\textit{0.073}}
& {\it 0.041} \\
& inserted
& \textbf{\textit{0.404}}
& {\bf 0.616}
& 0.475 \\
& discarded
& {\it 0.580}
& {\bf 0.311}
& {\it 0.484} \\
\hline

\multirowcell{3}{add\\gate}
& global best
& 0.023
& 0.064
& 0.040 \\
& inserted
& 0.321
& {\it 0.582}
& 0.445 \\
& discarded
& 0.655
& {\it 0.354}
& 0.515 \\
\hline

\multirowcell{3}{disable\\gate}
& global best
& \textbf{\textit{0.028}}
& {\it 0.065}
& {\bf 0.056} \\
& inserted
& {\it 0.400}
& 0.576
& 0.459 \\
& discarded
& \textbf{\textit{0.572}}
& 0.360
& 0.485 \\
\hline

\multirowcell{3}{clone}
& global best
& {\it 0.026}
& 0.063
& \textbf{\textit{0.064}} \\
& inserted
& {\bf 0.400}
& \textbf{\textit{0.643}}
& \textbf{\textit{0.522}} \\
& discarded
& {\bf 0.574}
& \textbf{\textit{0.294}}
& \textbf{\textit{0.414}} \\
\hline

\multirowcell{3}{qubit\\swap}
& global best
& 0.012
& 0.052
& 0.023 \\
& inserted
& 0.380
& 0.565
& {\bf 0.493} \\
& discarded
& 0.608
& 0.383
& {\bf 0.483} \\
\hline

\multirowcell{3}{enable\\gate}
& global best
& 0.019
& {\bf 0.070}
& 0.030 \\
& inserted
& 0.346
& 0.504
& {\it 0.483} \\
& discarded
& 0.634
& 0.426
& 0.488 \\
\hline

\multirowcell{3}{n-ary\\crossover}
& global best
& 0.015
& 0.017
& 0.027 \\
& inserted
& 0.194
& 0.259
& 0.220 \\
& discarded
& 0.791
& 0.724
& 0.753 \\
\hline

\end{tabular}
}

\caption{
\label{tab:image_operator_performance}
Performance rates for each mutation and crossover operation on the
MNIST, Fashion-MNIST, and CIFAR-10 image-classification tasks.
Best performance (high for global-best and inserted offspring, low for discarded offspring) is shown in bold italics, second best in bold, and third best in italics.
}
\end{table*}

\begin{figure}
	\centering
	\begin{subfigure}{0.32\linewidth}
		\includegraphics[width=\linewidth]{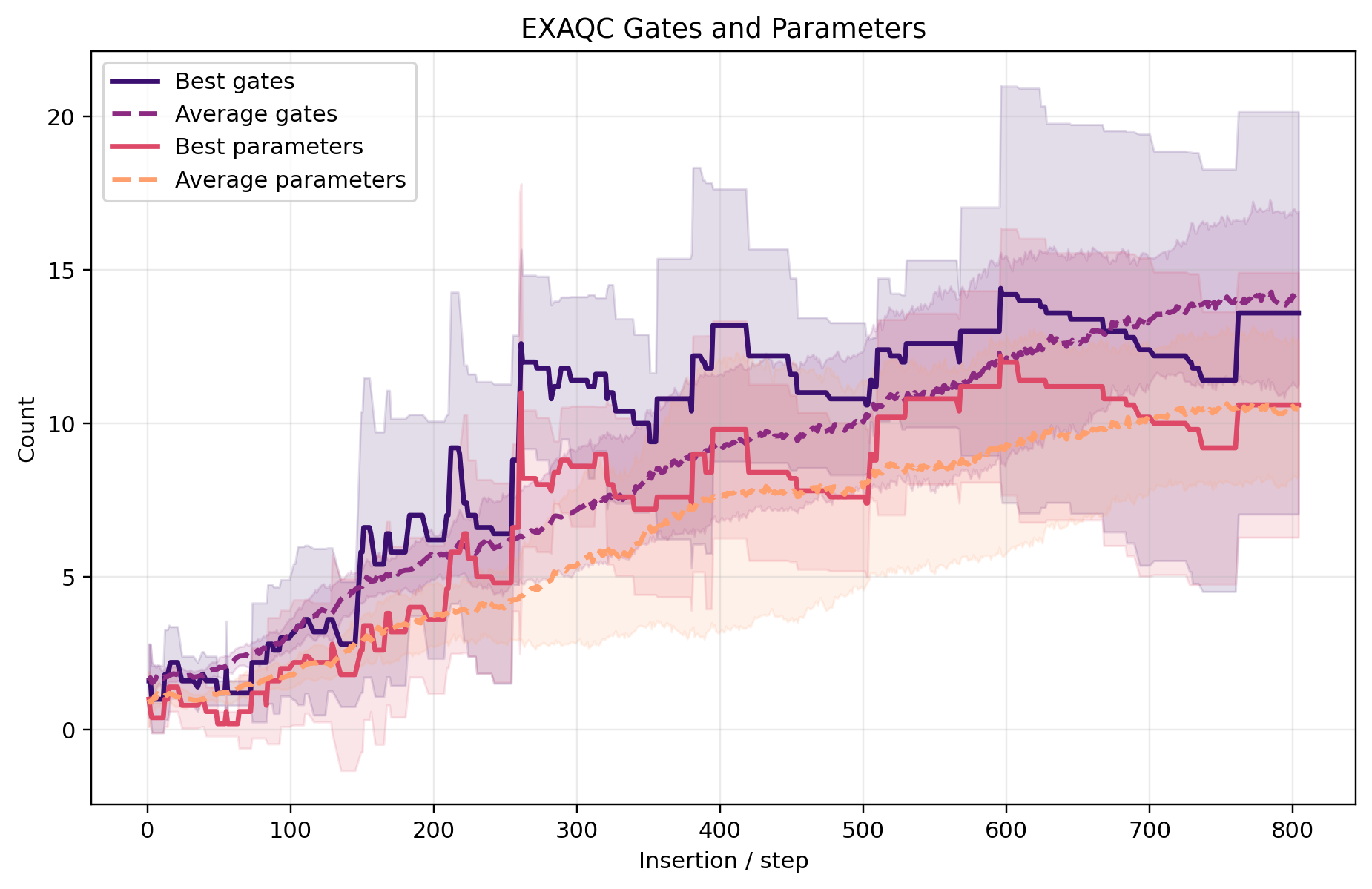}
		\caption{MNIST}
		\label{fig:mnist_gate_param}
	\end{subfigure}
	\begin{subfigure}{0.32\linewidth}
		\includegraphics[width=\linewidth]{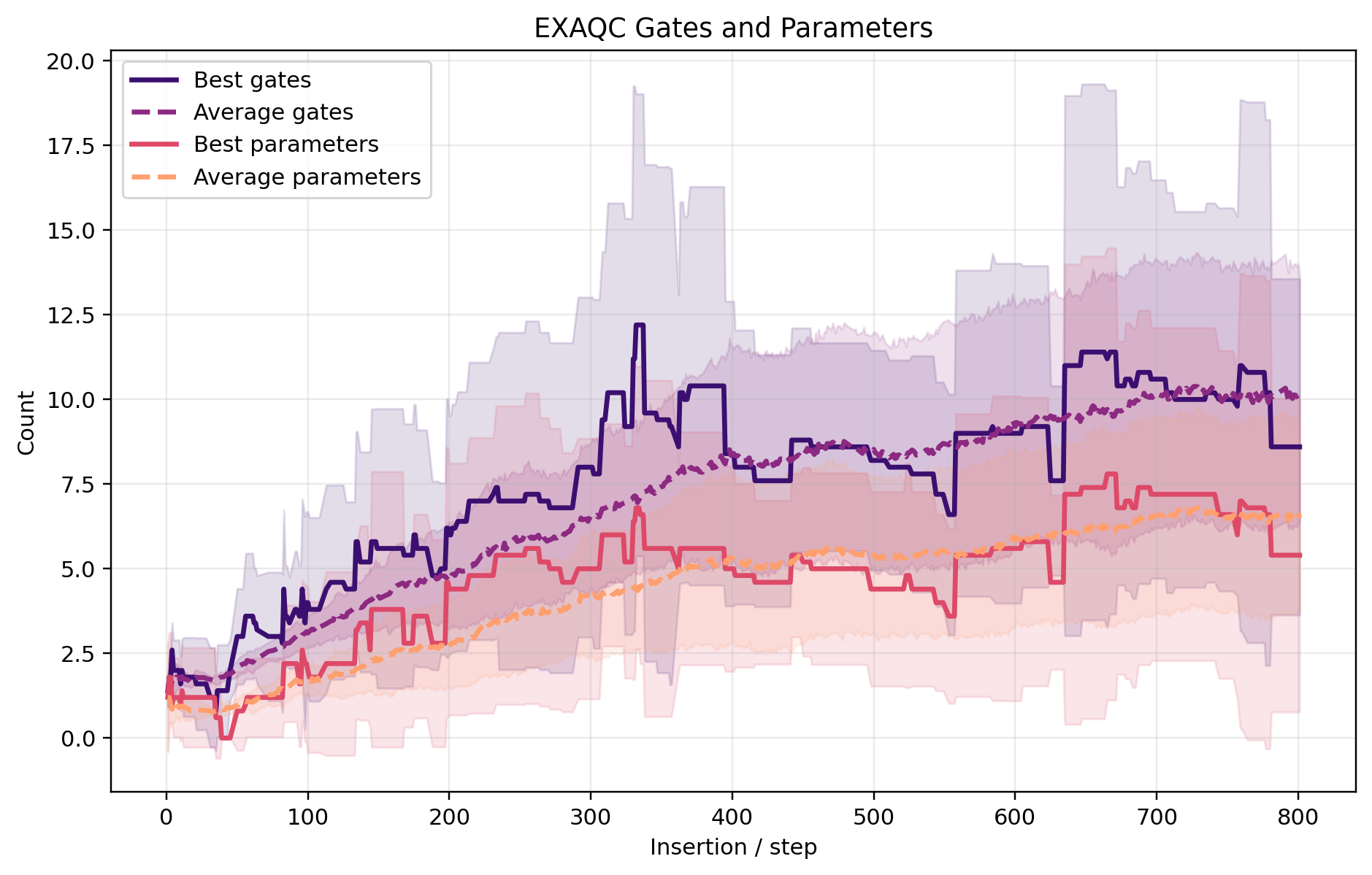}
		\caption{Fashion-MNIST}
		\label{fig:fashion_mnist_gate_param}
	\end{subfigure}
	\begin{subfigure}{0.32\linewidth}
	        \includegraphics[width=\linewidth]{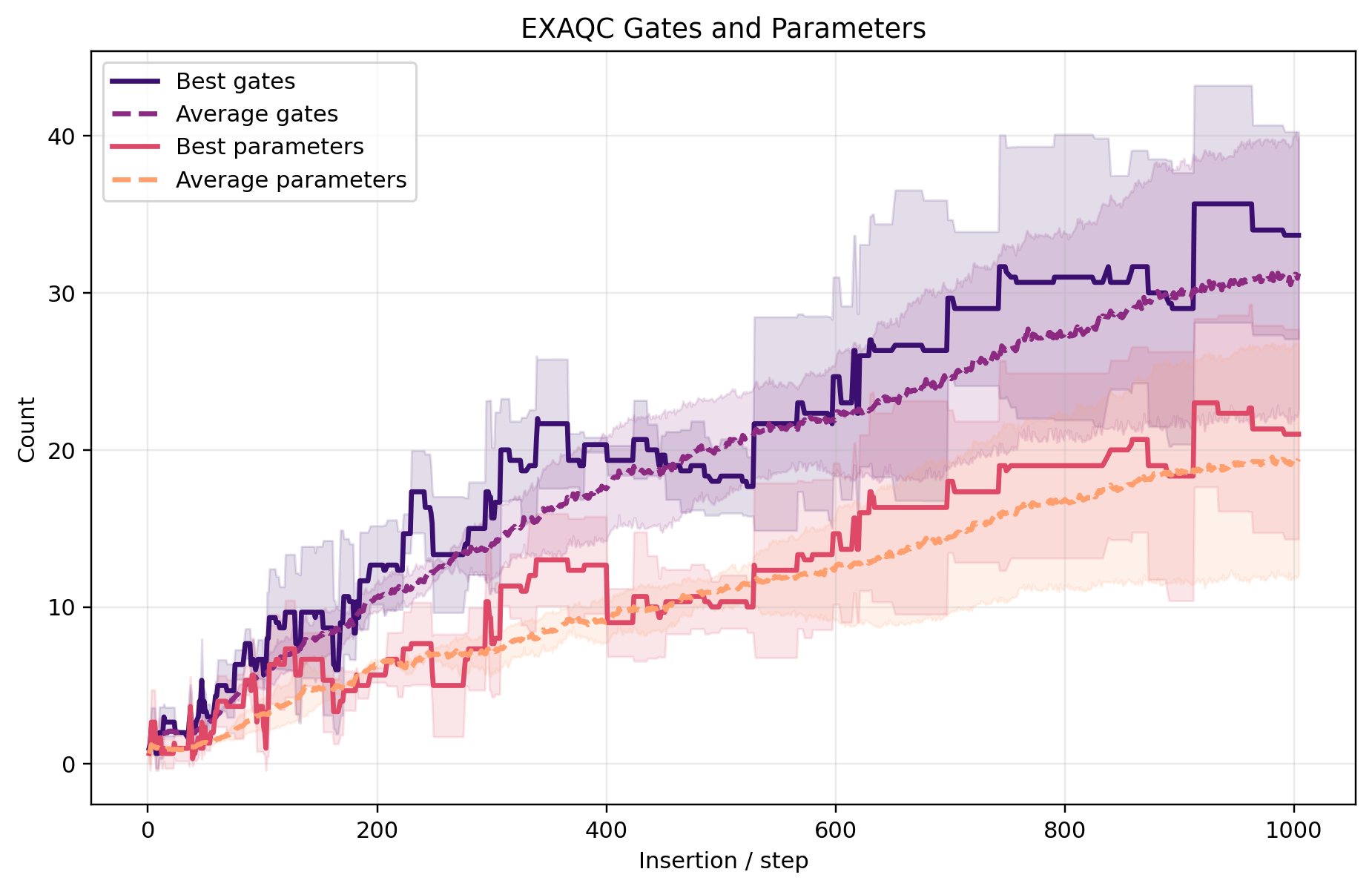}
	        \caption{CIFAR-10}
	        \label{fig:cifar10_gate_param}
         \end{subfigure}
	\caption{Evolution of quantum circuit complexity during EXAQC search on (a) MNIST, (b) Fashion-MNIST, and (c) CIFAR-10. Solid lines show the gate and parameter counts of the best-so-far circuits, while dashed lines show the average counts of generated circuits. Shaded regions indicate variation across independent runs.}
	\label{fig:gate_params}
\end{figure}

Figure~\ref{fig:gate_params} shows that circuit complexity increases rapidly early in evolution before stabilizing across all three datasets. Gate and parameter counts fluctuate rather than growing monotonically, consistent with EXAQC's gate-adding and gate-removing operators. Overall, EXAQC dynamically adapts circuit complexity while continuing to explore compact architectures. Table~\ref{tab:image_operator_performance} highlights complementary roles among the evolutionary operators. \textit{Reorder-gate} achieves the highest insertion rate on MNIST, while \textit{clone} dominates insertion on Fashion-MNIST and CIFAR-10. However, frequently accepted operators are not necessarily those most responsible for performance improvements: \textit{disable-gate} produces the highest global-best rate on MNIST, and \textit{reorder-gate} on Fashion-MNIST. On CIFAR-10, \textit{clone} leads in all three categories, suggesting that preserving and propagating strong circuits is particularly valuable on the more complex task. Crossover operators are generally more disruptive, with \textit{n-ary crossover} exhibiting the highest discard rate on all three datasets, whereas \textit{reorder-gate} and \textit{qubit-swap} effectively modify existing architectures. Overall, these results show that EXAQC benefits from heterogeneous operators that respectively preserve solutions, adjust circuit complexity, refine circuit structure, and promote broader exploration.

\subsection{Circuit Visualization}\label{subsec:circuit_viz}

\paragraph{Evolved MNIST Architecture}
Figure~\ref{fig:mnist_genome} shows a evolved hybrid architecture instance for MNIST, the simplest of the three. No convolutional layers are used: the flattened 784-pixel input passes through a two-layer MLP encoder ($784\!\rightarrow\!128\!\rightarrow\!6$), and the resulting 6 features are RY-encoded onto 6 qubits. The evolved circuit is compact, containing only seven gates, all parameterized: three IsingZZ interactions, two phase rotations, one RZ, and one controlled-RX. Probability readout over the 6 qubits yields 64 values, which a linear decoder maps to the 10 classes. Notably, under computational-basis readout, only the gates that change measurement probabilities contribute to the output. Every gate here except the controlled-RX is diagonal in the computational basis, so only the controlled-RX and the IsingZZ gate acting on its target qubit affect the output distribution.

\begin{figure}[!h]
    \centering
    \includegraphics[width=\linewidth]{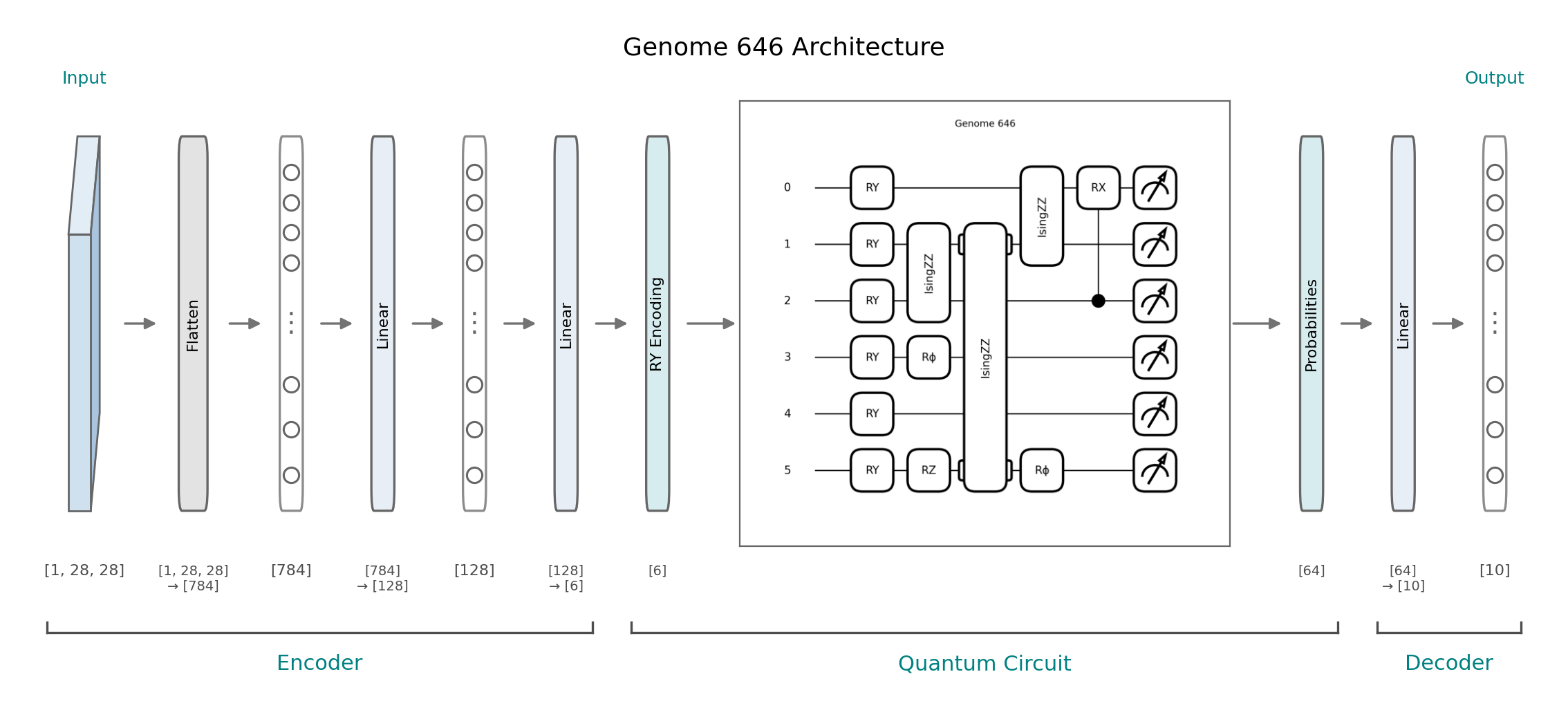}
    \caption{MNIST Genome Representation}
    \label{fig:mnist_genome}
\end{figure}

\paragraph{Evolved Fashion-MNIST Architecture}
Figure~\ref{fig:fashion_mnist_genome} shows a evolved hybrid architecture instance for Fashion-MNIST. A single convolutional block ($1\!\rightarrow\!32$ feature maps with $2\times2$ max pooling) is followed by two linear layers that compress the features to only 5 dimensions, a reduction of roughly $157\times$ from the 784-pixel input. These features are RY-encoded onto 5 qubits, one feature per qubit. The evolved circuit mixes parameterized two-qubit interactions (IsingYY, IsingZZ, and a controlled phase rotation) with fixed Clifford and non-Clifford gates (H, Z, SX, T, CNOT, CZ, and controlled-$Y$ and controlled-$H$). Probability readout over all 5 qubits yields 32 values, which a linear decoder maps to the 10 classes.

\begin{figure}[!h]
    \centering
    \includegraphics[width=\linewidth]{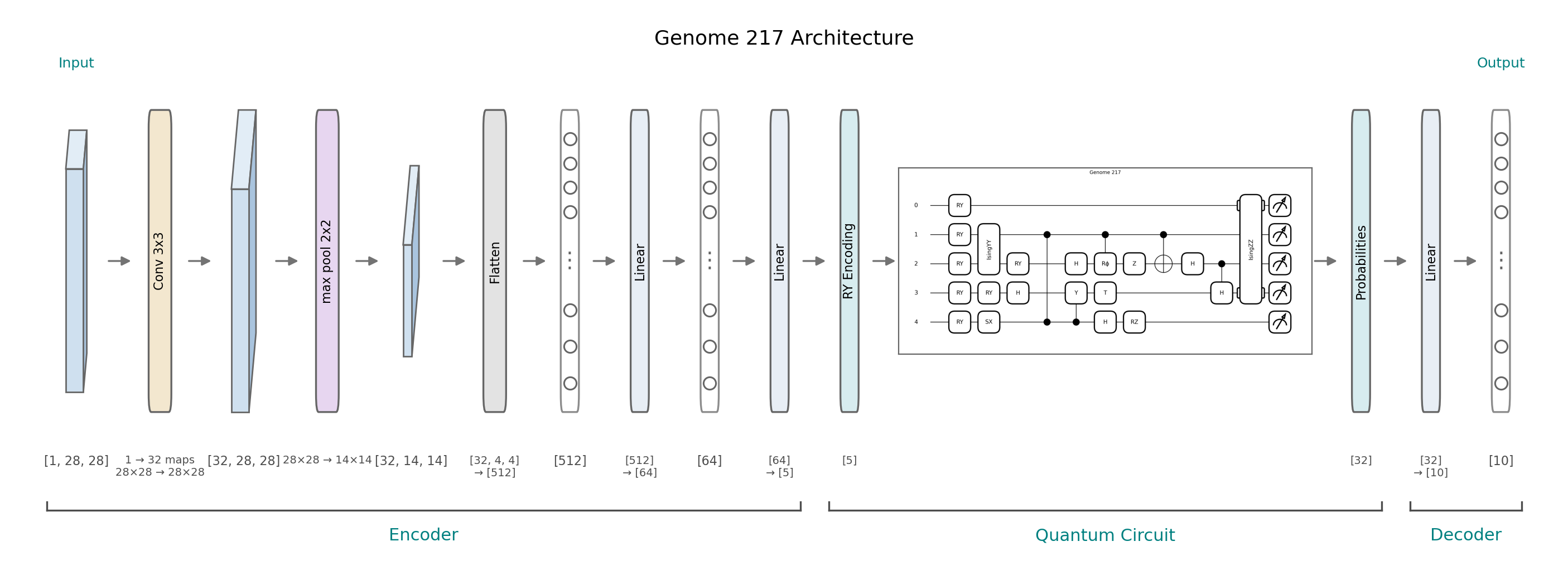}
    \caption{Fashion MNIST Genome Representation}
    \label{fig:fashion_mnist_genome}
\end{figure}



\end{document}